\documentclass[accepted]{article} 
\usepackage[dvipsnames]{xcolor}
\usepackage{fancyhdr}
\usepackage{icml2023}
\usepackage[disable,textsize=tiny]{todonotes} 
\usepackage{microtype}
\usepackage{graphicx}
\usepackage{booktabs}

\usepackage{xspace}

\newcommand{\eg}{{e.g.}\xspace}
\newcommand{\etc}{{etc.}\xspace}
\newcommand{\ie}{{i.e.}\xspace}

\newcommand{\wrt}{{w.r.t.}\xspace}

\usepackage{amsthm}
\newtheorem{theorem}{Theorem}
\newtheorem{corollary}{Corollary}
\newtheorem{example}{Example}
\newtheorem{remark}{Remark}
\newtheorem{definition}{Definition}

\usepackage{amsmath,amssymb,dsfont}

\def\1{\bm{1}}

\DeclareMathAlphabet{\mathsfit}{\encodingdefault}{\sfdefault}{m}{sl}
\SetMathAlphabet{\mathsfit}{bold}{\encodingdefault}{\sfdefault}{bx}{n}

\newcommand{\softmax}{\mathrm{softmax}}

\newcommand{\inner}[2]{\left\langle #1,#2 \right\rangle}

\newcommand{\zero}{\mathbf{0}}
\newcommand{\rank}{\mathrm{rank}}

\DeclareMathOperator{\sign}{sign}

\newcommand{\EE}{\mathds{E}}

\newcommand{\wv}{\mathbf{w}}
\newcommand{\xv}{\mathbf{x}}
\newcommand{\yv}{\mathbf{y}}
\newcommand{\zv}{\mathbf{z}}

\newcommand{\bv}{\mathbf{b}}
\newcommand{\gv}{\mathbf{g}}

\newcommand{\RR}{\mathds{R}}

\newcommand{\Dtr}{\mathcal{D}_{tr}}

\newcommand{\Dpo}{\mathcal{D}_{p}}

\usepackage[english]{babel}
\usepackage[utf8]{inputenc}
\usepackage{csquotes}
\usepackage[backend=biber,style=authoryear-comp,giveninits=true,uniquename=false,maxbibnames=9,uniquelist=false,maxcitenames=2,defernumbers=true,url=false,doi=false,isbn=false]{biblatex}

\renewbibmacro{in:}{%
	\ifentrytype{article}{}{%
		\printtext{\bibstring{in}\intitlepunct}}}
\renewbibmacro*{volume+number+eid}{%
	\printfield{volume}%
	\setunit*{\addnbspace}
	\printfield{number}%
	\setunit{\addcomma\space}%
	\printfield{eid}}
\renewbibmacro*{volume+number+eid}{%
	\setunit*{\addcomma\space}
	\printfield{volume}%
	\setunit*{\addcomma\space}
	\printfield{number}%
	\setunit{\addcomma\space}%
	\printfield{eid}
}
\DeclareFieldFormat[article]{volume}{\bibstring{volume}~#1}
\DeclareFieldFormat[article]{number}{\bibstring{number}~#1}

\newbibmacro{string+doiurlisbn}[1]{%
  \iffieldundef{doi}{%
    \iffieldundef{url}{%
      \iffieldundef{isbn}{%
        \iffieldundef{issn}{%
          #1%
        }{%
          \href{http://books.google.com/books?vid=ISSN\thefield{issn}}{#1}%
        }%
      }{%
        \href{http://books.google.com/books?vid=ISBN\thefield{isbn}}{#1}%
      }%
    }{%
      \href{\thefield{url}}{#1}%
    }%
  }{%
    \href{http://dx.doi.org/\thefield{doi}}{#1}%
  }%
}
\DeclareFieldFormat*{title}{\usebibmacro{string+doiurlisbn}{\mkbibquote{#1}}} 
\usepackage{wrapfig}
\usepackage[colorlinks,linkcolor=black,citecolor=black,urlcolor=black,bookmarks=true]{hyperref}

\usepackage{subfig}
\usepackage{multirow}
\usepackage{thm-restate}
\usepackage[ruled,vlined,linesnumbered]{algorithm2e}
\usepackage{booktabs}  
\usepackage{enumitem}
\usepackage[capitalize,noabbrev]{cleveref}

\usepackage{pifont}
\usepackage{fontawesome}
\usepackage{comment}
\usepackage{lipsum}
\usepackage{xargs} 

\setlist{nolistsep}
\renewcommand{\epsilon}{\varepsilon}

\usepackage{mathrsfs}
\newcommand{\Gfk}{\mathds{G}}
\newcommand{\Pcal}{\mathcal{P}}
\newcommand{\Wcal}{\mathcal{W}}
\newcommand{\Xds}{\mathds{X}}

\newcommand{\Zds}{\mathds{Z}}

\newcommand{\av}{\mathbf{a}}

\newcommand{\hv}{\mathbf{h}}
\newcommand{\mv}{\mathbf{m}}
\newcommand{\pv}{\mathbf{p}}

\newcommand{\uv}{\mathbf{u}}
\newcommand{\gfk}{\mathds{J}}
\newcommand{\Wfk}{\mathscr{W}}
\newcommand{\frev}[1][f]{\rotatebox[origin=c]{180}{$#1$}}

\newcommand{\citep}{\parencite}
\newcommand{\citet}{\textcite}

\defbibfilter{appendixOnlyFilter}{
  segment=1 
  and not segment=0 %
}

\icmltitlerunning{Exploring the Limits of Model-Targeted Indiscriminate Data Poisoning Attacks}

\begin{document}

\twocolumn[
\icmltitle{Exploring the Limits of Model-Targeted Indiscriminate Data Poisoning Attacks}

\icmlsetsymbol{equal}{*}

\begin{icmlauthorlist}
\icmlauthor{Yiwei Lu}{UW,VI}
\icmlauthor{Gautam Kamath}{UW,VI}
\icmlauthor{Yaoliang Yu}{UW,VI}
\end{icmlauthorlist}

\icmlaffiliation{UW}{School of Computer Science, University of Waterloo, Canada}
\icmlaffiliation{VI}{Vector Institute}

\icmlcorrespondingauthor{Yiwei Lu}{yiwei.lu@uwaterloo.ca}

\icmlkeywords{Machine Learning, data poisoning}

\vskip 0.3in
]
\printAffiliationsAndNotice{Authors GK and YY are listed in alphabetical order.}

\begin{abstract}
Indiscriminate data poisoning attacks aim to decrease a model's test accuracy by injecting a small amount of corrupted training data. 
Despite significant interest, existing attacks remain relatively ineffective against modern machine learning (ML) architectures.
In this work, we introduce the notion of \emph{model poisoning reachability} as a technical tool to explore the intrinsic limits of data poisoning attacks towards target parameters (i.e., model-targeted attacks). We derive an easily computable threshold to establish and quantify a surprising phase transition phenomenon among popular ML models: data poisoning attacks can achieve certain target parameters only when the poisoning ratio exceeds our threshold. Building on existing parameter corruption attacks and refining the Gradient Canceling  attack, we perform extensive experiments to confirm our theoretical findings, test the predictability of our transition threshold, and significantly improve existing indiscriminate data poisoning baselines over a range of datasets and models. 
Our work highlights the critical role played by the poisoning ratio, and sheds new insights on existing empirical results, attacks and mitigation strategies in data poisoning. Our code is available at \url{https://github.com/watml/plim}.
\end{abstract}

\section{Introduction}
\label{sec:intro}

Modern machine learning (ML) models require a large amount of training data to perform well on various tasks. Such hunger for data not only increases the training cost but also introduces potential risks during the data collection process \citep{NelsonBCJRSSTX08,SzegedyZSBEGF13,KumarNLMGCSX20}. Data poisoning, where an adversary can actively inject corrupted data into dataset aggregators or passively place poisoned samples on the web for scraping~\citep{GaoBBGHFPHTN20,Wakefield16, ShejwalkarHKR21,LyuYY20}, has caused serious concerns in the ML community and inspired a number of interesting works to expose and address this threat~\citep{Goldblumetal21}.

By now many data poisoning algorithms have been proposed; see \Cref{sec:bg} for some pointers. However, in the setting of indiscriminate data poisoning, where an attacker aims to decrease the overall test accuracy by adding a small fraction of corrupted data, the effectiveness of existing attacks remains underwhelming. For example, the recent work of \citet{LuKY22} achieved 1.11\% accuracy drop for a three-layer CNN on MNIST and a 5.54\% accuracy drop for ResNet-18 on CIFAR-10, after adding $\epsilon_d = 3\%$ poisoned data and retraining. 
Part of the difficulty lies in the computational challenge: the attacker has to anticipate what would happen after retraining the model on the mixed data (clean in-house data plus poisoned data). Other empirical works seem to suggest there might also be some intrinsic barrier to data poisoning; see \Cref{sec:bg} for a detailed discussion.

In this work we focus on model-targeted attacks \citep[\eg,][]{KohSL18,SuyaMSET21} and introduce the notion of \emph{model poisoning reachability}, \ie, given (arbitrary) clean training data, what model, represented by its parameter $\wv$, can be achieved through data poisoning, and what is the minimum (relative) percentage $\epsilon_d$ of poisoned data that one has to introduce, with what algorithm? 
While model poisoning reachability intuitively depends on the clean training data, the loss and the target model we aim to achieve, we show that under mild conditions, it can be largely characterized by a simple threshold $\tau$ that is readily computable and involves no training at all.
In particular, when the poisoning percentage $\epsilon_d$ falls under $\tau$, no algorithm could achieve the target model by retraining on a mixed dataset (however crafted). On the flip side, if $\epsilon_d > \tau$, we show that Gradient Canceling (GC), a refinement of the KKT attack of \citet{KohSL18}, can achieve a given target model surprisingly efficiently. 
We further demonstrate that most ML classifiers exhibit a phase transition: they become poisoning reachable only when $\epsilon_d$ crosses the threshold $\tau$.
In contrast, regression methods can be poisoning reachable even when $\epsilon_d$ approaches 0.
Thus, our results expose the critical role played by the poisoning percentage $\epsilon_d$, and clarify the somewhat disparate empirical results in the literature (with varying $\epsilon_d$).

Empirically, we apply the GC attack and verify the \emph{model poisoning reachability} property across a wide range of ML models, from logistic regression to residual networks on various datasets. 
Moreover, our work can also be applied as a distillation device: given any target parameter (namely the teacher, however crafted or impractical) that is effective for certain purpose, we can use our threshold and GC attack to pinpoint the (minimum) amount of poisoning data that needs to be constructed in order to simulate the teacher through retraining the model (student) over the combination of clean and poisoned data. 
Indeed, using the target parameters generated by parameter corruption  \citep{SunZRLL20} as a teacher, GC is able to construct more practical and effective (student) data poisoning attacks than baseline methods.

We summarize our main contributions as follows:
\begin{itemize}[leftmargin=*,topsep=0pt,itemsep=4pt]
    \item We formalize the notion of model poisoning reachability as a technical tool to study model-targeted data poisoning and we derive an easily computable
    threshold to characterize it. 
    
    \item We quantify the critical role played by the poisoning ratio $\epsilon_d$ and we establish a surprising phase transition for ML classifiers, explaining seemingly disparate empirical results obtained with varying $\epsilon_d$.
    
    \item We perform the Gradient Canceling attack on a number of models and datasets to extensively test our results. 
    With carefully chosen target parameters, we are able to improve existing indiscriminate data poisoning baselines.
\end{itemize}

\section{Background}
\label{sec:bg}

Data poisoning, an emerging concern on modern ML systems, refers to the threat of (often passively) crafting ``poisoned'' training data so that  systems retrained on it (along with possibly clean in-house data) are skewed towards certain behaviour. 
For example, indiscriminate data poisoning \citep[\eg,][]{BiggioNL12,KL17,KohSL18,GonzalezBDPWLR17,LuKY22} aims to decrease the overall test accuracy while targeted data poisoning \citep[\eg,][]{ShafahiHNSSDG18,AghakhaniMWKV20,GuoL20,ZhuHLTSG19e} only affects certain classes. Backdoor attacks \citep[\eg,][]{GuDG17,TranLM18,ChenLLLS17,SahaSP20} that aim to trigger a particular pattern, and unlearnable examples \citep[\eg,][]{LiuC10,HuangMEBW21,YuZCYL21,FowlGCGCG21, FowlCGGBCG21,SadovalSGGGJ22,FuHLST21} that aim to protect user data.


While many algorithms have been proposed for data poisoning, their effectiveness remains largely underwhelming against neural networks, especially when $\epsilon_d$, the relative proportion of poisoned data, is small. For example, Figure 4 of \citet{LuKY22} and Table 2 of \citet{HuangMEBW21} revealed that SOTA attacks can only decrease the test accuracy noticeably when $\epsilon_d$ is sufficiently (and sometimes exceedingly, e.g., $\epsilon_d>100\%$) large. These attacks, relying on sophisticated optimization tricks, are also rather expensive to run.
On the other hand, any data poisoning attack amounts to an \emph{indirect} way of rewiring an ML model (\ie, any change must be induced by retraining the model over clean and poisoned data). Direct approaches, such as the gradient-based parameter corruption (GradPC) attack of \citet{SunZRLL20,ZhangLRSLS21}, seek to overwrite a target model \emph{directly} (\ie, without constructing any poisoned data or retraining), under a perturbation constraint specified by $\epsilon_w$, \ie, the relative change of the model parameter should be less than $\epsilon_w$. While the applicability of direct approaches may seem limited, they are suitable for exploring the limits of more realistic data poisoning attacks.

In \Cref{tab:comparison} we compare the performance of the direct approach GradPC \citep{SunZRLL20} and the indirect approach TGDA \citep{LuKY22}. The latter adds $\epsilon_d = 3\%$ poisoned data while both attacks yield comparable perturbations of the (clean) model, as measured by $\epsilon_w$. 
The difference is significant, and begs the obvious question: what caused this difference? Is it because existing data poisoning attacks are not sufficiently optimized yet, or is there some intrinsic barrier to produce certain target parameters through data poisoning? To what extent would increasing $\epsilon_d$ help, and how do we know without trying every $\epsilon_d$? These questions will be formally and experimentally explored in the sequel, with the ultimate goal (if possible) to reduce the gap between data poisoning and parameter corruption attacks with comparable $\epsilon_w$, as highlighted in \Cref{tab:comparison}.
\begin{table}[t]
    \centering
    \small
    \caption{The attack accuracy/accuracy drop (\%) on  MNIST.}
    \vspace{-1em}
    \label{tab:comparison}    
    \setlength\tabcolsep{5pt}
    \scalebox{0.84}{\begin{tabular}{cccrr}
\toprule

\multirow{2}{*}[-.6ex]{\bf Model}
 & Clean & \bf TGDA & \multicolumn{2}{c}{\bf GradPC}\\ 
\cmidrule(l{-1pt}r{2pt}){2-3}\cmidrule(l{2pt}r{0pt}){4-5}
 & Acc. & Accuracy/Drop & $\epsilon_w=0.5$ &$\epsilon_w=1$ \\

\midrule
LR & 92.35 & 89.56 / 2.79 ($\epsilon_w =2.45$) & 69.80 / 22.55 & \bf 21.48 / 70.87 \\

NN & 98.04 & 96.54 / 1.50 ($\epsilon_w =0.55$) & 76.51 / 20.03 & \bf 31.14 / 66.90\\

CNN & 99.13 & 98.02 / 1.11 ($\epsilon_w =0.74$) & 73.24 / 24.78 & \bf 12.98 / 86.15  \\
\bottomrule
\end{tabular}
}
\vspace{-15pt}
\end{table}

\textbf{Connection with Learning Theory:} 
There has been significant work on training-time robustness in the learning theory literature, primarily focused on poisoning \emph{worst-case distributions}. Two models of robust PAC learning \citep{FrenayV13,NatarajanDRT13}, slightly rephrased for the sake of comparison, include the malicious noise model, where the adversary adds points \citep[\eg,][]{KearnsL88,CesabianchiDFSS99}, and the nasty noise model, where the adversary may both add and remove points \citep[\eg,][]{BshoutyEK02,BalcanBHS22}. 
Many of these theoretical results show strong computational barriers against robust learning for even the most basic problems.
Although our setting is similar to the malicious noise model (and we touch a bit on the nasty noise model in \Cref{app:replace}), there are three major differences with the majority of the theory literature: (1) our attacks address distributions that arise in practice, which differ from worst-case distributions; (2) while other attacks flip labels, we consider ``clean label'' attacks which are not visibly mislabeled; (3) we focus on model-targeted attacks whose goal is to induce certain target parameters while the above-mentioned references focus directly on decreasing accuracy on the test sample.
\section{Theoretical Results}

 \label{sec:method}

In this section we formalize the notion of model poisoning reachability as a technical tool for studying model-targeted data poisoning. We further derive an easily computable threshold $\tau$ and reveal that model-targeted data poisoning attacks are effective only when $\epsilon_d$, the (relative) percentage of poisoning data, crosses $\tau$.

\textbf{Notation and Preliminaries.} Let $\ell(\zv, \wv)$ be our loss that measures the fitness of our model $\wv$ on data $\zv \in \Zds$, \eg, $\zv = (\xv, y)$ for supervised learning and $\zv = \xv$ for unsupervised learning. Let $\Pcal = \Pcal(\Zds)$ denote the set of all distributions on $\Zds$, and we abstract the training data as an (empirical) distribution\footnote{For convenience in this work we do not distinguish the (clean) training set from the training distribution, \ie, $\mu$ can be empirical.} $\mu\in \Pcal$. For any given model $\wv$ and training distribution $\mu$, is it possible to construct a poisoning set, denoted by another (empirical) distribution $\nu$, such that $\wv$ minimizes $\ell$ over the mixed distribution $\chi = (1-\lambda)\mu + \lambda \nu$, where $\lambda = \tfrac{\epsilon_d}{1+\epsilon_d} \in [0,1]$ is the proportion of poisoning data.
To account for possible nonconvexity of the loss $\ell$, we relax the optimality of a model $\wv$ to simply have vanishing (sub)gradient. More formally, let 
\begin{align}
\label{eq:grad-point}
\gv(\zv) = \gv(\zv; \wv) = \nabla_{\wv} \ell(\zv; \wv)  
\end{align}
be the gradient vector with respect to a fixed model $\wv$ evaluated at the data $\zv$. 
For practical reasons (\eg, to evade possible defenses or to account for the technical capabilities of an attacker) we also restrict the poisoning distribution $\nu$ into a convex subset $\Gamma \subseteq \Pcal$ of admissible distributions. For instance, we may consider 
\begin{align}
\Gamma = \Gamma_{\mu, \delta} := \{\gamma: \|\gamma - \mu\| \leq \delta\}, 
\end{align}
where $\|\cdot\|$ 
denotes (say) the Wasserstein distance. 
By definition we always have $\mu \in \Gamma$. 
For each $\nu \in \Gamma$, define 
\begin{align}
\gv(\nu) = \gv(\nu; \wv) := \EE_{\zv \sim \nu} \gv(\zv; \wv),
\end{align}
i.e., the average gradient \wrt the distribution $\nu$. Clearly, 
\begin{align}
\label{eq:grad-set}
\Gfk = \Gfk(\Gamma) := \{\gv(\nu): \nu \in \Gamma\}
\end{align}
is a subset of the closed convex hull of all gradient vectors. In fact, equality holds when $\Gamma=\Pcal$ (\eg, $\delta = \infty$). 

\subsection{Model Poisoning Reachability}
We can now state our fundamental problem of interest:
\begin{definition}[Model Poisoning Reachability] 
We say a target parameter $\wv$ is $(\ell, \mu, \Gamma, \lambda)$-poisoning reachable
if there exists some poisoning distribution $\nu\in\Gamma$ such that
\begin{align}
\gv(\chi; \wv) 
= (1-\lambda) \gv(\mu; \wv) + \lambda \gv(\nu; \wv) = \zero, 
\end{align}
\ie, the parameter $\wv$ has vanishing gradient (\wrt loss $\ell$) over the mixed distribution $\chi = (1-\lambda) \mu + \lambda \nu$.
\label{def:mp}
\end{definition}
When the loss $\ell$, training distribution $\mu$, and admissible poisoning distributions $\Gamma$ are evident, we will simply say the parameter $\wv$ is $\lambda$-poisoning reachable, or poisoning reachable if it is $\lambda$-poisoning reachable for some $\lambda \in [0,1]$.

We make three further remarks regarding \Cref{def:mp}: (a) If we are interested in more quantitative results about data poisoning, for example, is it possible to craft a poisoning set such that retraining on the mixed distribution would decrease test accuracy by a large margin, we need only specify a set of target models $\wv \in \Wcal$ that all decrease the test accuracy as required\footnote{As pointed out by a reviewer, this may not be computationally feasible if one is too ambitious about the set $\Wcal$.}, and we say data poisoning is successful if any $\wv \in \Wcal$ is ($\lambda$-) poisoning reachable.
(b) \Cref{def:mp} leaves out the computational aspects of data poisoning, \ie, how efficiently we can find such a poisoning distribution $\nu$ (whenever it exists). This will be studied in \Cref{sec:gc}, using a gradient-based algorithm inspired directly by our definition. 
(c) We could also add other requirements, such as curvature or stability, to \Cref{def:mp}.

Given the above formalization, the following characterization is immediate: 
\begin{theorem}
A target parameter $\wv$ is $\lambda$-poisoning reachable iff 
$
\zero \in \Gfk^\lambda = \Gfk^\lambda(\gv(\mu)) := \{(1 - \lambda) \gv(\mu) + \lambda \gv: \gv \in \Gfk\}. 
$    
\label{thm:main}
\end{theorem}
Since $\Gfk$ (see equations~\eqref{eq:grad-point}-\eqref{eq:grad-set}) is clearly convex, the subsets $\Gfk^\lambda$ are all convex
and increasing with respect to $\lambda$, \ie, 
\begin{align}
\gv(\mu) = \Gfk^0 \subseteq \Gfk^\lambda \uparrow \subseteq \Gfk^1 = \Gfk. \nonumber
\end{align}
Recall that $\lambda = \tfrac{\epsilon_d}{1+\epsilon_d}$ 
is the (absolute) proportion of the poisoned set. Thus, we conclude intuitively that the larger $\epsilon_d$ (equivalently $\lambda$) is, the easier it is to induce any target model $\wv$ on any training distribution $\mu$. 
In particular, the special case $\lambda = 1$ corresponds to the so-called ``unlearnable examples'' \citep{LiuC10,HuangMEBW21,YuZCYL21,FowlGCGCG21, FowlCGGBCG21,SadovalSGGGJ22,FuHLST21}, where an attacker is allowed to change the entire training set (\ie, empirical distribution $\mu$).

Conversely, we can also conclude from \Cref{thm:main} that if $\zero \not \in \Gfk(\Gamma)$, then data poisoning, with any budget $\epsilon_d$, will not be successful in producing the target parameter $\wv$. If $\zero \not \in \Gfk(\Pcal)$, then no training distribution can yield $\wv$. In particular, data poisoning will not be successful in producing $\wv$ even if $\epsilon_w = \infty$.

Let us give some examples to illustrate our results so far. 
\begin{example}[Least-square regression] 
Consider 

\vspace{-1.9em}
\begin{align}
\ell(\zv; \wv) = \tfrac12 (y - \wv^\top \xv)^2, \mbox{ where }\  \zv = (\xv, y).\nonumber
\end{align}
\vspace{-1.9em}

Clearly, we have 

\vspace{-1.9em}
\begin{align}
    \gv(\xv, y) &= (\wv^\top \xv - y) \xv = (\xv \xv^\top) \wv - y \xv, \nonumber
\end{align}
\vspace{-1.9em}

and hence $\gv(\mu) = \Sigma \wv - \mv$,
where $\Sigma = \EE_{\xv \sim \mu}\xv\xv^\top$ and $\mv = \EE_{(\xv, y) \sim \mu} y\xv$. 
For simplicity let us assume $\Zds = \RR^d \times \RR$ and $\Gamma = \Pcal$ so that $\Gfk = \RR^d$ (by considering product distributions where $\xv$ concentrates on a single point). 
Therefore, we conclude from \Cref{thm:main} that data poisoning with any $\epsilon_d > 0$ is possible for least-square regression. 
The same conclusion holds even if we add regularization to $\wv$ (which, we recall, is fixed).
\label{exm:ls}
\end{example}
\vspace{-.5em}
\subsection{Scalar Output Linear Models} 
For linear models we can further simplify the iff condition in \Cref{thm:main}. 
We begin with the following result: 
\begin{theorem}
Suppose $\Gamma = \Pcal$ contains all distributions, $\ell((\xv, y); \wv) = l(\wv^\top \xv, y)$ for some univariate loss $l$, and $\inner{\wv}{\gv(\mu)} \ne 0$. Then, $\wv$ is $\lambda$-poisoning reachable iff 
\begin{align}
\label{eq:linearthreshold}
\lambda &> \max\left\{\tfrac{\inner{\wv}{\gv(\mu)}}{\inner{\wv}{\gv(\mu)} -a}, ~  \tfrac{-\inner{\wv}{\gv(\mu)}}{b - \inner{\wv}{\gv(\mu)}} \right\}, \quad \mbox{where} 
\\
a &= \inf_{(\xv, y)\in \Zds}  (\wv^\top\xv) \cdot l'(\wv^\top\xv, y),\nonumber \\
b &= \sup_{(\xv, y) \in \Zds}  (\wv^\top\xv) \cdot l'(\wv^\top\xv, y),\nonumber
\end{align}
with equality attained if the maximum is attained.
\label{thm:lr}
\end{theorem}
\Cref{thm:lr} follows from the more general \Cref{thm:lm} in \Cref{sec:proof}, where we further remove the restriction $\Gamma = \Pcal$.
The condition $\inner{\wv}{\gv(\mu)} \ne 0$ can be checked easily \emph{a priori}; see \Cref{rem:deg} (\Cref{sec:proof}) for discussions on when it fails. 
\Cref{rem:breakdown} (\Cref{sec:proof}) draws further connection between our result and the breakdown point in robust statistics. 
Compared to the more general \Cref{thm:main}, \Cref{thm:lr} exploits the linear structure to simplify the set $\Gfk$ to basically an interval and hence the condition \eqref{eq:linearthreshold} is much easier to verify. Indeed, consider \Cref{exm:ls} again. It is clear that $l'(t, y) = t-y$, whence $a=-\infty$ and $b=\infty$. Thus, we verify more easily that data poisoning succeeds on least-square regression for any $\epsilon_d > 0$. 

\begin{figure*}
    \centering    
    \includegraphics[width=1.0\textwidth]{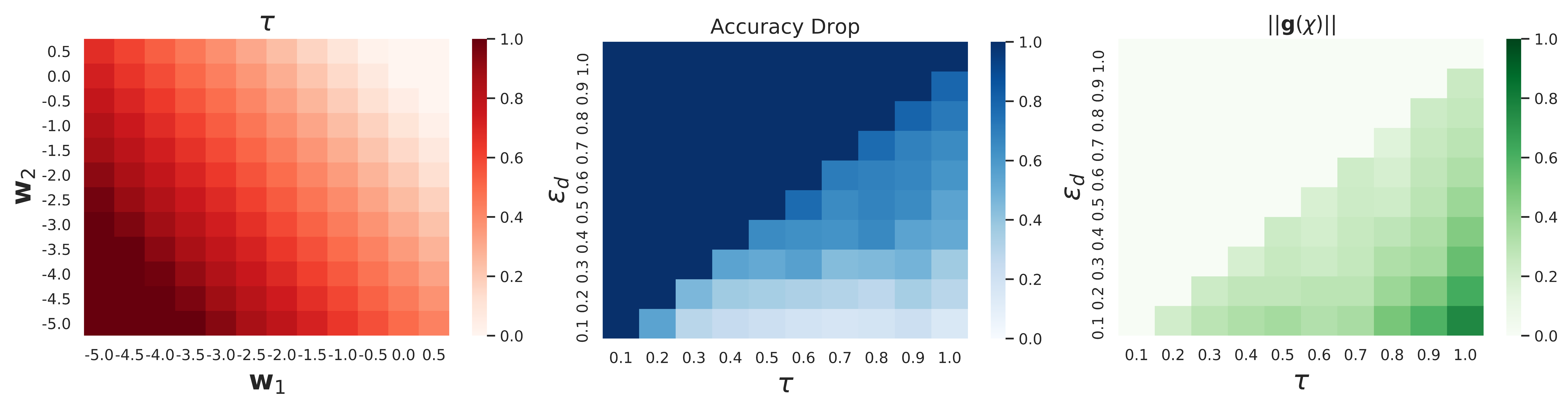}
    \vspace{-2em}
    \caption{Logistic regression on the 2d OR dataset that verifies the transitioning threshold $\tau$ in \Cref{thm:linear}. 
    \textbf{Left}: $\tau$ \wrt target models $\wv\in\RR^2$, which all achieve 0 accuracy; \textbf{Middle}: accuracy drop due to the gradient canceling attack in \Cref{sec:gc}. Indeed, poisoning successfully induces the target model $\wv$ as long as $\epsilon_d\geq \tau$;  \textbf{Right}: norm of gradient \wrt model $\wv$ over the mixed distribution $\chi$, with $\epsilon_d$ the relative proportion of poisoned data. In general, the closer $\epsilon_d$ gets above $\tau$, the smaller the gradient norm, which is an indication of the target model being more achievable through data poisoning.}
    \vspace{-1em}
    \label{fig:heatmap_or} 
\end{figure*}

The next example reveals a surprising phase transition in terms of the poisoning proportion $\lambda$ (or equivalently $\epsilon_d$):
\begin{example}[Logistic regression]
    Consider now 
\begin{align}
\ell(\zv; \wv) = \log(1+\exp(-\wv^\top\tilde\xv)),\nonumber
\end{align}
where we have absorbed the binary label $y$ into $\tilde\xv$ (\eg, $\tilde\xv \gets y \xv$).
Clearly, we have 
$    \gv(\tilde\xv) = -\tfrac{1}{1+\exp(\wv^\top\tilde\xv)} \tilde\xv$.
On the direction $\wv$, for any distribution $\mu$ we have 
\begin{align}
   \!\!-\Wfk(\tfrac1e) \!=\! \inf_{t} \tfrac{-t}{1+\exp(t)} \leq \inner{\wv}{\gv(\mu)} \leq \sup_{t} \tfrac{-t}{1+\exp(t)},\nonumber
\end{align}
where the left-hand side is Lambert's W function and the right-hand side is clearly $\infty$. Therefore, suppose $\Xds = \RR^d$ and $\Gamma = \Pcal$, we have 
\begin{align}
    \Gfk = \{\gv: \wv^\top \gv \geq - \Wfk(1/e) \approx -0.28\},\nonumber
\end{align}
which is not the entire space! Consequently, if 
\begin{align}
    \lambda < \tfrac{\inner{\wv}{\gv(\mu)}}{\inner{\wv}{\gv(\mu)} + \Wfk(1/e)} \iff \epsilon_d < \tau := \max\{\tfrac{\inner{\wv}{\gv(\mu)}}{\Wfk(1/e)}, 0\},   \label{eq:transition}
\end{align}
then any poisoning distribution $\nu$ (with any support) cannot 
produce $\wv$ (along with training distribution $\mu$)!
\label{example:log}
\end{example}

By simply changing $\tilde\xv \gets y\xv$ and then dropping $y$ we immediately obtain from \Cref{thm:lr} sufficient and necessary conditions for the poisoning reachability of binary margin classifiers. In particular, we record the following result:
\begin{corollary}[Binary Margin Classifier]
Consider linear models with loss $\ell(\tilde\xv; \wv) = l(\wv^\top \tilde\xv)$. 

Suppose $\Gamma = \Pcal$ consist of all distributions on $\tilde\Xds$ and $\inner{\wv}{\gv(\mu)}\ne 0$. 
Define 
\begin{align}
a := \inf_{t\in \wv^\top\tilde\Xds} t \cdot l'(t), \quad
b := \sup_{t\in \wv^\top\tilde\Xds} t \cdot l'(t).\nonumber
\end{align} 
Then, 
a target parameter $\wv$ is $\lambda$-poisoning reachable iff 
\eqref{eq:linearthreshold} holds (with equality attained if the maximum there is attained).
\label{thm:linear}
\end{corollary}
The standard margin losses are decreasing, such as the logistic loss in \Cref{example:log}, the exponential loss in Adaboost, and the hinge loss in SVM. 
When $\Xds = \RR^d$ is unbounded, typically $b=\infty$ but $a > -\infty$, leading to a common phase transition phenomenon: data poisoning against these losses succeeds in producing a target parameter $\wv$ iff $\lambda$ crosses the threshold in \eqref{eq:linearthreshold}. In particular, any target parameter $\wv$ such that $\inner{\wv}{\gv(\mu)} < 0$ is always poisoning reachable for any $\lambda > 0$. 
Interestingly, \citet[Proposition 3]{KohSL18} showed that if a model is poisoning reachable, then it (often) can be poisoned to by a distribution $\nu$ supported on two distinct points (which however does not imply diminishing $\epsilon_d$ due to repetitions). \Cref{thm:linear} provides a definitive answer on when a model is poisoning reachable and hence complements the results of \citet{KohSL18}.

We emphasize that with any further restrictions on the poisoning distribution (such that $\Gamma \subsetneq \Pcal$), condition \eqref{eq:linearthreshold} remains to be necessary: data poisoning is apparently even harder in this case. 
For nonlinear models with a fixed feature map $\phi$ (such as kernel methods), our results extend immediately, after the obvious change-of-variable $\xv \gets \phi(\xv)$.

\Cref{fig:heatmap_or} illustrates the transition threshold $\tau$ in \eqref{eq:transition} on the simple OR dataset (where each of the four points is repeated 50 times with small Gaussian perturbation, see \Cref{app:add_exp} for details). Logistic regression (LR), trained on the clean data, achieves perfect accuracy. In \Cref{fig:heatmap_or} (left), each grid point represents a target parameter $\wv = (w_1, w_2)$, all of which achieve 0 test accuracy (\ie, malicious models). The heat map indicates the threshold $\tau$ for each $\wv$, which, as predicted by our theory, is the percentage of poisoning required to achieve $\wv$ through retraining. In \Cref{fig:heatmap_or} (middle) we run the gradient canceling attack (see \Cref{sec:gc}) with varying percentage $\epsilon_d$ and verify that indeed we can reduce the 100\% clean accuracy to 0\% iff $\epsilon_d \geq \tau$. In \Cref{fig:heatmap_or} (right) we plot the magnitude of the gradient of the target parameter $\wv$ over the mixed dataset (clean training data plus poisoned data), as an approximate measure of how close $\wv$ can be achieved by retraining on the mixed dataset.
Overall, the larger $\epsilon_d$ is, the larger the accuracy drop is (not surprisingly) and the smaller the gradient norm is, with a clear transition once $\epsilon_d$ crosses $\tau$ (perhaps surprisingly).

\subsection{Multiple Output Linear Models}
Next, we extend our results to multiple outputs (classes):
\begin{restatable}[Multiclass]{theorem}{LMC-sim}
Consider $\ell(\xv, \yv; W) = l(W^\top\xv, \yv)$ for some loss $l$. 
Then\footnote{We use the notation $\av \otimes \bv := \av \bv^\top$ for two column vectors.}, 
\begin{align}
\label{eq:mc-gradw}
    G(\xv, \yv) := \nabla_{W}\ell(\xv, \yv; W) &= \xv \otimes \nabla l(W^\top\xv, \yv).
\end{align}
Suppose $W^\top G(\mu)$ is non-degenerate and $\Gamma = \Pcal$ contains all distributions.  
Then, $W$ is $\lambda$-poisoning reachable
iff 
\begin{align}
\label{eq:mc-mp2}
    \zero \in (1-\lambda) W^\top G(\mu) + \lambda \{W^\top G(\nu): \nu \in \Gamma\}.
\end{align}
\label{thm:lmc}
\end{restatable}
\vspace{-2em}

Compared to \Cref{thm:lm}, condition \eqref{eq:mc-mp2} is no longer univariate but a square matrix of dimensions the same as $\yv$ (the output). 
Nevertheless, we may simply take the trace on both sides to arrive at an easier albeit only necessary condition. 
We illustrate the last point through a familiar example:
\begin{example}[Cross-entropy] 
Let $\hv = W^\top \xv$. 
The cross-entropy loss corresponds to 
\begin{align}
l(\hv, \yv) &= -\inner{\hv}{\yv} + \log \sum\nolimits_k \exp h_k,\nonumber
\end{align}
where $\yv$ is one-hot. Taking trace on \eqref{eq:mc-mp2} we obtain
\begin{align}
0 &= (1-\lambda) g(\mu) + \lambda g(\nu), \mbox{ where } \nonumber\\
g(\nu) &= \EE_{(\xv,\yv) \sim \nu} \inner{\hv}{\pv-\yv},\nonumber
\end{align}
and $\pv := \softmax(\hv) = \exp(\hv) / \sum_k \exp(h_k)$.
In \Cref{sec:proof} we prove the tight bound
$
   -\Wfk(\tfrac{c-1}{e}) 
   \leq g(\nu) \leq  
   \infty$,
leading to the necessary condition for inducing $W$:
\begin{align}
\label{eq:tau}
\epsilon_d \geq \tau = \tau(c) := \max\{\inner{W}{G(\mu)}/\Wfk(\tfrac{c-1}{e}), 0\}, 
\end{align}
where $c$ is the number of classes. When $c=2$, we recover the sufficient and necessary condition in \eqref{eq:transition}.
\label{exm:ce}
\end{example}

We remark that all of our results continue to hold as necessary (but may not be sufficient) conditions for neural networks where the input $\xv$ goes through a \emph{learned} feature transformation $\varphi(\xv; \uv)$, parameterized by $\uv$: 
\begin{restatable}[Neural Networks]{theorem}{NN}
Consider $\ell(\xv, \yv; W, \uv) = l(\hv, \yv)$ for some loss $l$, where $\hv := W^\top \varphi(\xv; \uv)$. 
Then, 
\begin{align}
\label{eq:nn-gradw}
    \nabla_{W}\ell(\xv, \yv; W, \uv) &= \varphi(\xv;\uv) \otimes \nabla_{\hv} l(\hv, \yv) \\    
\label{eq:nn-gradu}     
    \nabla_{\uv}\ell(\xv, \yv; W, \uv) &= \nabla_{\uv}\varphi(\xv;\uv) W \nabla_{\hv} l(\hv, \yv),
\end{align}
and $(W, \uv)$ is $\lambda$-poisoning reachable iff there exists $\nu \in \Gamma$ such that
\begin{align}
\label{eq:nn-grad}
    \zero \in (1-\lambda) G(\mu) + \lambda G(\nu),    
\end{align}
where $G(\nu) := \EE_{(\xv, \yv)\sim\nu} \left(\nabla_{W}\ell, \nabla_{\uv}\ell \right)$.
In particular, $(W, \uv)$ is $\lambda$-poisoning reachable only if there exists some $\nu\in\Gamma$ such that 
\begin{align}
\label{eq:nn-grad-W}
    \zero \in (1-\lambda) G_1(\mu) + \lambda G_1(\nu),
\end{align}
where $G_1(\nu) := \EE_{(\xv,\yv) \sim \nu} \varphi(\xv; \uv) \otimes \nabla_{\hv} l(\hv, \yv)$.
\label{thm:nn}
\end{restatable}
\section{Gradient Canceling Attack}
\label{sec:gc}
In this section we discuss how to find a poisoning distribution $\nu\in\Gamma$ so that upon retraining on the mixed distribution $\chi = (1-\lambda)\mu + \lambda \nu$, the target parameter $\wv$ will be favored. We recall that $\mu$ is the (clean) training distribution and $\lambda$ is the (absolute) poisoning proportion.  

The algorithm we propose is very intuitive and directly inspired by our \Cref{def:mp}: we simply find a poisoning distribution $\nu\in\Gamma$ so that 
\begin{align}
\gv(\chi) = (1-\lambda) \gv(\mu) + \lambda \gv(\nu) \approx \zero,
\end{align}
where recall that $\lambda = \tfrac{\epsilon_d}{1+\epsilon_d}$ and 
$\gv(\nu) := \EE_{\zv \sim \nu} \nabla_{\wv} \ell(\zv; \wv)$
is the model gradient computed over a distribution $\nu$. Thus, we arrive at the following Gradient Canceling problem\footnote{Other merit functions than the $\ell_2$-norm here can also be used.}: 
\begin{align}
\label{eq:gc}
    \min_{\nu \in \Gamma} ~ \tfrac12\|\gv(\mu) + \epsilon_d \gv(\nu) \|_2^2,
\end{align}
which is always convex (since $\gv(\nu)$ is linear in $\nu$ while $\Gamma$ is a convex subset of admissible distributions). In \Cref{sec:measure} we discuss a measure optimization approach for solving \eqref{eq:gc}, while below we focus on a Lagrangian approach that directly constructs a poisoning dataset and eliminates the need of resampling from $\nu$.

In more details, we constrain the poisoning distribution 
to be uniform over $n\epsilon_d$ data points $\{\zv_j\}$: 
\begin{align}
\label{eq:pdata}
\hat\nu = \frac{1}{n\epsilon_d} \sum_{j=1}^{n\epsilon_d} \delta_{\zv_j},
\end{align}
where $n$ is the size of the (clean) training set and $\delta_{\zv}$ denotes the point mass concentrated on $\zv$. We only optimize the locations of the points $\zv_j$ but keep their mass uniform throughout. 
 
Thus, we arrive at the following problem:
\begin{align}
    \label{eq:loss_function}
    \min_{ \hat\nu\in\Gamma } ~\frac12 \Big\| \gv(\mu) + \epsilon_d \cdot \frac{1}{n\epsilon_d}\sum_{j=1}^{n\epsilon_d} \nabla_{\wv} \ell(\zv_j; \wv) \Big\|_2^2,
\end{align}
where we remind that $\gv(\mu) = \EE_{\zv\sim\mu}\nabla_{\wv} \ell(\zv; \wv)$ as well as the target parameter $\wv$ are fixed during optimization. 
For supervised tasks where $\zv = (\xv, \yv)$, we may choose to optimize both the feature $\xv$ and label $\yv$, or simply optimize the feature $\xv$ only (as in our experiments).

We apply (projected) gradient descent to solve \eqref{eq:loss_function}, where the gradient with respect to the $j$-th poisoning data $\zv_j$ is:
\begin{align}
\label{eq:gc-grad}
    \frac{\partial}{\partial \zv_j} = \frac{1}{n}
    \nabla_{\zv} \nabla_{\wv} \ell(\zv_j; \wv) \cdot [\gv(\mu) + \epsilon_d \gv(\hat\nu)].
\end{align}
We note that using auto-differentiation, the above matrix vector product can be computed very efficiently, costing essentially as much as gradient calculation. The constraint for $\hat\nu$ to lie in $\Gamma$ can be handled by projection. For instance, the constraint $\zv\in\Zds$ (e.g. pixels must lie in $\Zds = [0,1]$) can be enforced by projecting the gradient update onto $\Zds$.

We summarize the Gradient Canceling(GC) attack in \Cref{alg:gc}, and we emphasize that it can take \emph{any} target parameter $\wv$ as ``teacher'' and construct a poisoning dataset such that retraining will arrive (approximately) at $\wv$. 
We note that Gradient Canceling is a refinement of the KKT attack of \citet{KohSL18}: our refinement lies in the generalization to any loss $\ell$, different optimization strategy, exploring target parameters generated by the much stronger GradPC attack \citep{SunZRLL20}, experimenting on a variety of different models, and studying the effect of the poisoning proportion.  
Other authors such as \citet{SuyaMSET21} also explored (rather costly) attacks based on a target parameter in the online setting (that require retraining in each round), whereas their lower bound on the amount of poisoned points may not be easily computed even for logistic regression.

\begin{algorithm}[t]
\DontPrintSemicolon
    \KwIn{training distribution $\mu$,  

    step size $\eta$, 
    poisoning fraction $\epsilon_d$, and target parameter $\wv$.}
    
    initialize poisoned dataset $\hat\nu$ in \eqref{eq:pdata}, \eg, randomly subsample clean training data
    
    calculate $\gv(\mu) = \EE_{\zv \sim \mu} \nabla_{\wv}  \ell(\zv; \wv)$ 
    
    \For{$t =1, 2, ...$}{
    calculate $\gv(\hat\nu) \gets \tfrac{1}{n\epsilon_{d}}\sum_{j=1}^{n\epsilon_{d}} \nabla_{\wv} \ell(\zv_j; \wv) $
    
    calculate loss $\mathcal{L} = \tfrac12\|\gv(\mu) + \epsilon_d \gv(\hat\nu)\|_2^2 $
    
    update poisoned data using \eqref{eq:gc-grad}: $\zv_j \gets \zv_j - \eta \frac{\partial\mathcal{L}}{\partial \zv_j} $ 

    project to admissible set: $\hat\nu \gets \mathrm{Proj}_{\Gamma}(\hat\nu)$
    }
    
\textbf{return} the final poisoned dataset $\hat\nu$
   
\caption{Gradient Canceling(GC) Attack}
\label{alg:gc}
\end{algorithm}

\textbf{Comparison with Gradient Matching.}
\citet{GeipingFHCTMG20} proposed a gradient matching algorithm for crafting \emph{targeted} poisoning attacks, which can be easily adapted to our setting. Suppose 
that a defender aims at minimizing a loss $\ell$ to achieve model $\wv$ on (clean) training distribution $\mu$. Let \frev[\ell] be a reversed version of $\ell$. For example, if $l$ is the cross-entropy loss in \Cref{exm:ce}, then 
\begin{align} 
\!\!\!\frev[l](\hv, \yv) \!&= \!-\!\log[1\!-\!\exp(-l(\hv,\yv))], \mbox{ where } \hv \!=\! W^\top\xv,\!\!
\end{align} 
is the reversed cross-entropy loss \citep{FowlCGGBCG21}. 
As \frev[\ell] discourages the model from classifying clean data $\xv$ as $\yv$, \citet{GeipingFHCTMG20} proposed to match its gradient $\nabla_{\wv} \frev[\ell](\mu, \wv)$ over a poisoned distribution $\hat\nu$ (within some proximity of $\mu$), based on some dissimilarity function $\mathcal{S}$ (\eg, cosine dissimilarity):
\begin{align}
    \min_{\hat\nu \in \Gamma} ~\mathcal{S}\left(\nabla_{\wv} \frev[\ell](\mu; \wv),~ \nabla_{\wv} \ell(\hat\nu; \wv)\right).
\end{align}
We point out some key differences between gradient matching  \citep{FowlCGGBCG21} and our work: (1) Gradient matching focuses  on $\lambda = 1$, \ie, an attacker is able to modify the entire training set. While this is useful in certain settings (\eg, crafting ``unlearnable examples''), it masks the effect of the poisoning proportion, which, as we showed in \Cref{sec:method}, can determine if a target parameter is poisoning reachable at all. (2) Gradient matching requires the construction of a reversed loss, whose gradient may not be at the same scale as that of the loss we are interested in. Thus, one typically can only hope to align the direction of gradients, which does not necessarily imply the desired matching in performance. In contrast, \Cref{alg:gc} only requires the original loss and our theory gives guidance on when it succeeds.
(3) There is no guarantee that after retraining over $\hat\nu$, gradient matching will arrive at the target parameter while  \Cref{alg:gc} explicitly aims to achieve this goal.
Further experimental comparisons against gradient matching will be presented in \Cref{sec:exp} and \Cref{app:add_exp}.

\section{Experiments}
\label{sec:exp}

\begin{table*}[t]
    \centering
    \small
    \caption{The attack accuracy/accuracy drop (\%) on MNIST, CIFAR-10 and TinyImageNet. We perform GC based on the target parameters generated by GradPC. Our attack significantly outperforms TGDA and Gradient Matching.}
    \vspace{-1em}
    \label{tab:results}    
    \setlength\tabcolsep{5pt}
    \scalebox{0.88}{\begin{tabular}{cccc|rrrr|ccr|rrr}
\toprule

\multirow{2}{*}[-.6ex]{Dataset}& Target Model
 & Clean Acc & GradPC & \multicolumn{4}{c|}{\bf Gradient Canceling}   & \multicolumn{3}{c|}{\bf TGDA} & \multicolumn{3}{c}{\bf Gradient Matching}    \\ 
 & $\epsilon_d$ & 0 & 0 & 0.03 & 0.1 & 1 & $\epsilon_d=\tau$  & 0.03 & 0.1 & 1  & 0.03 & 0.1 & 1  \\

\midrule

\multirow{3}{*}{MNIST}
&LR & 92.35 & -70.87 ($\tau$=1.15) & \bf -22.97 & \bf -63.83 & \bf -67.01 & -69.66  &   -2.79 & -4.01 & -8.97 & -3.33  & -8.14 &  -12.13   \\

& NN & 98.04   & -20.03 ($\tau$=2.48) &  \bf -6.10 & \bf -9.77 & \bf -12.05 & -19.05  &  -1.50 & -1.72 & -5.49 & -2.82  &  -3.71 & -\phantom{0}4.03 \\

& CNN & 99.13 & -24.78 ($\tau$=0.98) &  \bf -9.55 & \bf -20.10 & \bf -23.80 &  -23.77  &  -1.11 & -1.31 & -4.76 &  -2.01  & -3.80 & -6.94 \\
\midrule

CIFAR-10 & ResNet-18 & 94.95 & -21.69 ($\tau$=1.29)& \bf -13.73 & \bf -16.40 & \bf -18.33 & -19.98  &  -5.54 &-6.28 & -17.21 & -6.01 & -7.62 & -9.80 \\
\midrule

TinyImageNet & ResNet-34 & 66.65 &   -24.77 ($\tau$=1.08) & \bf -13.22 & \bf -16.11 & \bf -20.15 & -22.79  & -4.42 & -6.52 & -14.33 & -5.53 & -7.72 & -10.85  \\
\bottomrule
\end{tabular}}
\vspace{-5pt}
\end{table*}

\begin{figure*}
    \centering    
    \includegraphics[width=1.0\textwidth]{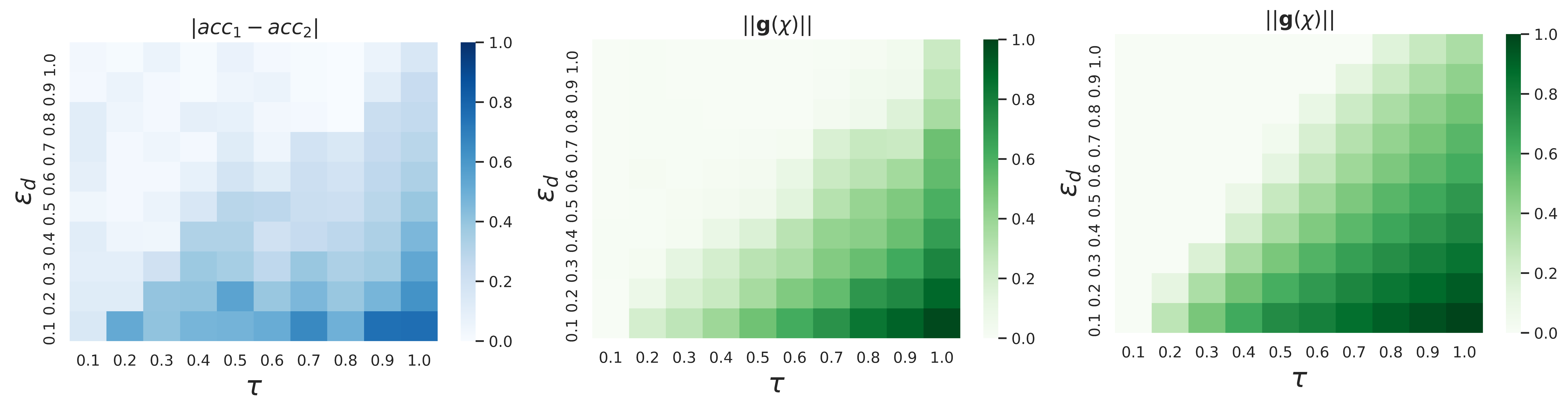}
    \vspace{-2.5em}
    \caption{We run experiments on logistic regression to verify the transition threshold $\tau$ in \Cref{thm:linear}. \textbf{Left}: accuracy difference between GC and GradPC on {10-d  Gaussian dataset}; \textbf{Middle}:  norm of the gradient over the mixed dataset $\chi$ on {10-d  Gaussian dataset}; \textbf{Right}: norm of the gradient over the mixed dataset $\chi$ on MNIST-17.}
    \vspace{-1em}
    \label{fig:heatmap} 
\end{figure*}

We perform extensive experiments to verify our main results: (a) how competitive the gradient canceling attack (\Cref{alg:gc}) is compared to SOTA baselines in indiscriminate data poisoning? (b) to what extent our threshold $\tau$ (see \eqref{eq:tau}) can predict model poisoning reachability?(c) how effective gradient canceling remains against certain existing defense mechanisms?

\subsection{Experimental Settings}

\textbf{Dataset:} We consider image classification on  MNIST~\citep{Deng12} (60k training and 10k test images),  CIFAR-10~\citep{Krishevsky09} (50k training and 10k test images), and TinyImageNet \citep{chrabaszcz2017downsampled} (100k training, 10k validation and 10k test images).
For the first two datasets, we further split the training data into 70\% training set 
and 30\% validation set,
respectively. 

\textbf{Target Models:} 
We examine the following ML models. On MNIST: Logistic Regression (LR), a fully connected neural network (NN) with three layers and a convolutional neural network (CNN) with two convolutional layers, max-pooling and one fully connected layer; On CIFAR-10: ResNet-18 \citep{HeZRS16}; and on TinyImageNet: ResNet-34.

\textbf{Baselines:} We compare to TGDA \citep{LuKY22} and Gradient Matching \citep{GeipingFHCTMG20} attacks. To our knowledge, the TGDA attack is one of the most effective data poisoning attacks against neural networks. Gradient Matching was originally proposed for targeted attacks and unlearnable examples, and we also compare against it due to its similarity with the Gradient Canceling (GC) attack.

\textbf{Implementation:} For GC implementation, we follow \Cref{alg:gc} 
and we discuss the effect of the projection step in \Cref{sec:defense}. Most of our target parameters are generated using GradPC \footnote{We follow the implementation in \url{https://github.com/TobiasLee/ParamCorruption}.}  except LR on MNIST, where we use $\epsilon_w=1$ to allow meaningful accuracy drop and transition threshold $\tau$ \footnote{We discuss the selection of target parameters in \Cref{app:tar}.}. We initialized the poisoned points with a random $n\epsilon_d$ sample from the clean training set and we only optimized the feature vectors but not the labels. Accuracy drops are obtained after retraining.

\textbf{Evaluation Protocol:} To evaluate the effectiveness of different attacks, we first apply each attack to acquire its poisoned set and then retrain the model from scratch (initialized with the same random seed across all attacks) on both clean and poisoned data until convergence. The (test) accuracy drop, compared with clean accuracy (obtained by training on clean data only), is reported across all experiments.

\begin{figure*}[t]
    \centering
    \includegraphics[width=\textwidth]{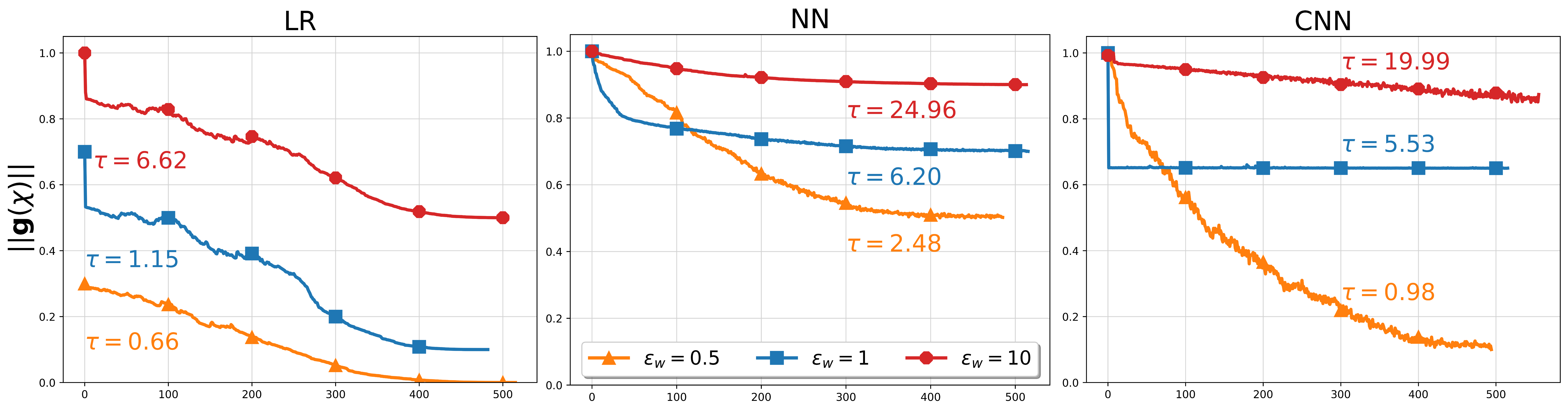}
    \vspace{-1.8em}
    \caption{The learning curve for running GC on MNIST with different target models $\wv$ and $\epsilon_w$. We fix $\epsilon_d=1$, and the curves indicate the decrease of the gradient $\|\gv(\chi)\|$ \wrt GC epoch. We  confirm that GC fails to achieve $\wv$ when $\epsilon_d < \tau$.}
    \vspace{-.8em}
    \label{fig:epsilon_w}
\end{figure*}

\subsection{How Competitive Is Gradient Canceling (GC)?}

\Cref{tab:results} reports the accuracy drop of LR, NN, CNN and ResNet due to GC on the aforementioned datasets. We note the trade-off of $\epsilon_w$ in GradPC when generating a target parameter $\wv$: the larger $\epsilon_w$ is, the more effective GradPC is but also the larger the resulting transition threshold $\tau$ is, meaning that GC (or any other data poisoning attack) can succeed (in reproducing $\wv$) only with a larger proportion $\epsilon_d$ of poisoned points. 
We used $\tau = \tau(2)$ in \Cref{tab:results} as we find it is much more indicative than the more conservative $\tau(c)$ (which is roughly 11 times smaller on TinyImageNet and 4 times smaller otherwise).

We observe that GC is much more effective than TGDA and Gradient Matching, across all datasets, models, and choices of $\epsilon_d$. This confirms that existing data poisoning attacks are under-optimized and there is room for future improvements. Moreover, when $\epsilon_d$ approaches the transition threshold $\tau$, GC, a \emph{bona fide} data poisoning attack, indeed achieves a comparable accuracy drop as GradPC (which directly overwrites the model). 
While \Cref{tab:comparison} still has room to improve, both in terms of the tightness of  $\tau$ and the effectiveness of GC, we believe our results yield significant insights on indiscriminate data poisoning, in particular the theoretical and experimental quantification of the detrimental effect of a large proportion $\epsilon_d$ of poisoned points.

\subsection{Predicting Poisoning Reachability Using $\tau$}

Next, we further examine the predictability of the transition threshold $\tau \approx \max\{3.6 \inner{W}{G(\mu)}, 0\}$, 
whose main term is simply proportional to the inner product between a target parameter and its gradient on the clean training data. 

\textbf{Binary Logistic Regression}: We have already shown the predictability of $\tau$ on the OR dataset in \Cref{fig:heatmap_or}. 
In \Cref{fig:heatmap} we show additional results on a 10-dimensional Gaussian dataset (see \Cref{sec:exp-add}) and MNIST-17 (consisting only of digits 1 and 7). 
The observations are similar: GC could achieve similar accuracy drops as GradPC (which directly overwrites the model), as long as $\epsilon_d$ crosses the threshold $\tau$.
We note that the threshold $\tau$ tends to be more conservative as the dimension of the problem increases, which we believe is largely because the optimization cost of GC becomes accordingly higher, making convergence harder to attain.

\textbf{Multi-class with Cross-Entropy:} 
We also perform experiments on multi-class problems with the cross-entropy loss in \Cref{exm:ce}. In \Cref{tab:results} we have confirmed that when $\epsilon_d > \tau$, GC largely achieves the target parameters generated by GradPC. We now further examine the opposite case where $\epsilon_d < \tau$. We fix $\epsilon_d =1$ and vary $\epsilon_w$ in GradPC, consequently generating target parameters with varying $\tau$ on MNIST. \Cref{fig:epsilon_w} shows how much the gradient $\|\gv(\chi)\|$ of the target parameters decreases \wrt each epoch of GC (when $\chi$, the mixed dataset, gets updated). We observe that the gradients do not converge to 0, indicating that GC failed to produce the target parameters. The failure of GC indicates that a larger poisoning proportion $\epsilon_d$ may be necessary to produce the target parameters, as confirmed by our theory.

\begin{table}

\addtolength{\tabcolsep}{-3pt}
    \centering
    \caption{Accuracy drop (\%) of Gradient Canceling (w/wo clipping) on MNIST against Sever defense (+ indicates the accuracy increased by the defense). GC-c: GC with clipped output; GC-d: GC after defense; GC-cd: GC-c after defense.}
    \vspace{-1em}
    \label{tab:defense}    
    \setlength\tabcolsep{5pt}
    \scalebox{0.77}{\begin{tabular}{crlrrrrrr}
\toprule

\multirow{2}{*}[-.6ex]{\bf Model} & \multirow{2}{*}[-.6ex]{Clean}& \multirow{2}{*}[-.6ex]{$\epsilon_d$} &\multirow{2}{*}[-.6ex]{GC} & \multirow{2}{*}[-.6ex]{GC-c}& \multicolumn{2}{c}{Sever}  \\
 
\cmidrule(l{2pt}r{0pt}){6-7}
& & & & & GC-d  & GC-cd \\

\midrule
\multirow{3}{*}[-.6ex]{LR} & \multirow{3}{*}[-.6ex]{92.35} & 0.03 &-22.79 & -11.28 & -12.81 / \bf +9.98 & -9.66 / \bf +1.62  \\

& & 0.1 & -63.83 & -26.77 & -59.79 / \bf +4.04 & -25.53 / \bf +1.24\\

& & 1 & -67.01 & -28.99 & -65.01 / \bf +2.00 & -27.89 / \bf +1.10  \\

\midrule

\multirow{3}{*}[-.6ex]{NN} & \multirow{3}{*}[-.6ex]{98.04} & 0.03 & -6.10 & -3.25 & -3.22 / \bf +2.88 & -2.26 / \bf +0.90   \\
& & 0.1 & -9.77 & -5.10  & -7.66 / \bf +2.11 & -4.46 / \bf +0.56\\
& & 1 & -12.05 & -6.53 & -10.02 / \bf +2.03 & -6.11 / \bf +0.42\\
\midrule

\multirow{3}{*}[-.6ex]{CNN} & \multirow{3}{*}[-.6ex]{99.13}  & 0.03 & -9.55 & -5.87  & -5.55 / \bf +4.00 & -4.36 / \bf +1.51  \\
& & 0.1 & -20.10 & -12.50 & -16.55 / \bf +3.55 & -11.32 / \bf +1.18 \\
& & 1 & -23.80 & -13.32  & -21.05 / \bf +2.75 & -12.51 / \bf +0.81 \\

\bottomrule
\end{tabular}}
\vspace{-15pt}
\end{table}

\vspace{-.3em}
\subsection{Does GC Remain Effective Against Defenses?}
\label{sec:defense}
Lastly, we choose several defenses from \citep{AngleEAZPLKTPS22} and examine the effectiveness of GC against (1) a distribution-wise certified defense Sever~\citep{DiakonikolasKKLSS19}, which removes $\epsilon_d$ training points with the highest outlier scores, defined using the top singular value of the gradient matrix, and (2) one of the SOTA pointwise certified defenses \citep{LevineF21,WangLF22a,WangLF22b}    called Deep Partition Aggregation (DPA) \citep{LevineF21}, which provides certified robustness for individual test samples. More results \wrt other defenses (e.g., influence defense and max-up defense) can be found in \Cref{sec:more-defenses}. 

\begin{table}

\addtolength{\tabcolsep}{-3pt}
    \centering
    \caption{The Certified Accuracy (CA) (\%) of DPA and Accuracy drop (\%) of Gradient Canceling ($\epsilon_d=0.008$) on MNIST against the DPA defense (+ indicates the accuracy increased by the defense).}
    \vspace{-.3em}
    \label{tab:defense_dpa}    
    \setlength\tabcolsep{5pt}
    \scalebox{0.8}{\begin{tabular}{ccccccc}
\toprule

\bf Model & Clean & $k$ & Clean (DPA) & GC & CA & DPA  \\

\midrule
\multirow{2}{*}[-.6ex]{LR} & \multirow{2}{*}[-.6ex]{92.35} & 1200 & 91.33 &-8.25 & 47.12 & -4.68/\bf +3.57   \\

& & 3000& 89.97 & -8.25 & 49.23 & -4.21/\bf +4.04 \\

\midrule

\multirow{2}{*}[-.6ex]{NN} & \multirow{2}{*}[-.6ex]{98.04} & 1200 & 94.65 & -2.25 & 46.11 & -1.29/\bf +0.96 \\
& & 3000 & 92.37 & -2.25 & 48.52 & -1.17/\bf+1.08  \\
\midrule

\multirow{2}{*}[-.6ex]{CNN} & \multirow{2}{*}[-.6ex]{99.13}  & 1200 & 95.53 & -2.77 & 47.22 & -1.66/\bf+1.11   \\
& & 3000 & 93.15 & -2.77 & 50.01 & -1.52/\bf+1.25\\

\bottomrule
\end{tabular}}
\vspace{-.8em}
\end{table}

\textbf{Results on Sever:}
\Cref{tab:defense} reports the accuracy drops on MNIST. We observe that 
(1) Sever indeed reduces the effectiveness of GC, consistently across all models.
(2) Clipping poisoned data to the range of the clean training set makes GC more robust against all defenses, at the cost of less effectiveness in terms of accuracy drop.
(3) Even with clipping and against defenses, GC still largely outperforms TGDA and Gradient Matching.
(4) Larger $\epsilon_d$ generally makes GC both more effective and more robust, which matches our observation in least-squares regression (see \Cref{sec:ls}).

\textbf{Results on DPA}: although DPA is originally proposed for pointwise robustness, it can be easily applied to the indiscriminate data poisoning setting. Here we choose $k=1200/3000$ for DPA and fix $\epsilon_d=0.008$ to roughly preserve median certified robustness on MNIST. Note that we choose the base classifiers to be the same as the target models. We report the certified accuracy (CA), which is the percentage of certified robust examples among the test set, and (relative) accuracy increase due to deploying DPA in \Cref{tab:defense_dpa}. We observe that DPA is generally effective against GC, where the (relative) accuracy increased by DPA roughly approaches its certified accuracy. For example, on LR with $k=3000$, GC was able to decrease test accuracy by $8.25\%$, whereas with the DPA defense, $4.04\%$ (relative) accuracy drop of GC are rectified, leading to an effectiveness that is roughly proportional to its certified accuracy, \ie, $4.04 \approx 8.25 * 0.4923$.

\section{Conclusion and Future Work}
\label{sec:con}

In this work, we introduce the notion of \emph{model poisoning reachability} as a technical tool to study the intrinsic limits in model-targeted data poisoning. We give complete characterizations on the poisoning ratio that any data poisoning attack has to satisfy (in order to induce a given target parameter), and we derive an easily computable threshold that is readily applicable and gives guidance on crafting effective model-targeted attacks. 
Using the gradient canceling attack, we perform extensive experiments on a number of datasets and models to quantify the critical role played by the poisoning ratio, confirm the precision of our transition threshold, and achieve better performance against existing baselines (w/wo several existing defenses). 
Our empirical results also reveal further room to sharpen the transition threshold and develop more effective data poisoning attacks, and we mention the exciting possibility of designing (clean) in-house data to mitigate and regulate the risk of future poisoning attacks. 

One limitation of this work is its focus on achieving specific target parameters, which may not always be available or necessary. Indeed, data poisoning attacks that are not based on any target parameter abound. However, we point out that our work may still be valuable for the latter class of attacks, for instance, as a distillation device: a data poisoning attack can use our threshold to evaluate the potential ``wastefulness'' of its constructed poisoning set (along with the model parameter obtained by retraining) and then use GC to further distill and improve it. 
Another limitation is that most existing data poisoning attacks, including GC, assume a lot of knowledge of the victim model (\eg, fixed architecture, access to clean training data, \etc) and hence may not always be realistic. 
Advanced and adaptive defense mechanisms may also thwart the effectiveness of many attacks (including GC). Further investigations of these issues form another important direction for future research.

\vspace{1mm}
\subsubsection*{Acknowledgments}
We thank the reviewers for the critical comments that have largely improved the presentation and precision of this paper.
We gratefully acknowledge funding support from NSERC and the Canada CIFAR AI Chairs program. 
Resources used in preparing this research were provided, in part, by the Province of Ontario, the Government of Canada through CIFAR, and companies sponsoring the Vector Institute.

\printbibliography[segment=0]

\newpage
\appendix
\onecolumn
\newrefsegment
\numberwithin{equation}{section}
\section{Proofs}
\label{sec:proof}

\begin{restatable}[Linear Models]{theorem}{LinearModel}
Consider $\ell((\xv, y); \wv) = l(\wv^\top \xv, y)$ for some univariate loss
$l$. Then, 
\begin{align}
    \gv(\xv, y) = \xv \cdot l'(\wv^\top \xv, y), \nonumber
\end{align}
and $\wv$ is $\lambda$-poisoning reachable iff there exists $\nu\in\Gamma$ such that 
\begin{align}
    0 \in (1-\lambda)\gv(\mu) + \lambda \gv(\nu).\nonumber
\end{align}
Suppose $\inner{\wv}{\gv(\mu)} \ne 0$. Consider $\Pi \subseteq \Pcal$ and  
let 
\begin{align}
    \!\! \gfk := \{ \EE_{(\xv, y) \sim \nu} (\wv^\top\xv) \cdot l'(\wv^\top\xv, y): \nu \in \Pi \} \subseteq \RR.\nonumber
\end{align}
Then, $\wv$ is $\lambda$-poisoning reachable
if \footnote{$T_{\#}\nu$ denotes the distribution of $T(\zv)$ when $\zv \sim\nu$.} $\Gamma \supseteq T_{\#}\Pi$ and 
\begin{align}
\label{eq:mp1}
    0 \in (1-\lambda) \inner{\wv}{\gv(\mu)} + \lambda \gfk, 
\end{align}
where the transformation $T(\xv, y) := \left(\tfrac{\inner{\wv}{\xv}}{\inner{\wv}{\gv(\mu)}} \gv(\mu), y \right)$.
Conversely, \eqref{eq:mp1} holds if $\wv$ is $\lambda$-poisoning reachable and $\Pi \supseteq \Gamma$.
\label{thm:lm}
\end{restatable}
\begin{proof}
The gradient computation is straightforward while the first claim follows from \Cref{thm:main}. 

Suppose now $\inner{\wv}{\gv(\mu)} \ne \zero$.

Suppose first \eqref{eq:mp1} holds, so we can choose $\nu \in \Pi$ such that 
\begin{align}
    0 = (1-\lambda) \inner{\wv}{\gv(\mu)} + \lambda \EE_{(\xv,y) \sim \nu} (\wv^\top\xv) \cdot l'(\wv^\top\xv , y).
\end{align} 
Consider the transformation $T(\xv, y) = \left(\tfrac{\inner{\wv}{\xv}}{\inner{\wv}{\gv(\mu)}} \gv(\mu), y\right)$ and let $\tilde\nu = T_{\#}\nu$, which is in $\Gamma$ due to our assumption $\Gamma \supseteq T_{\#}\Pi$. We then have 
\begin{align}
\EE_{(\tilde\xv, \tilde y) \sim \tilde\nu} l'(\wv^\top\tilde\xv, \tilde y) \tilde\xv = \EE_{(\xv, y) \sim \nu} l'(\wv^\top\xv, y) \tfrac{\inner{\wv}{\xv}}{\inner{\wv}{\gv(\mu)}} \gv(\mu),
\end{align}
and hence 
\begin{align}
(1-\lambda) \gv(\mu) + \lambda \EE_{(\tilde\xv,\tilde y) \sim \tilde\nu} \nabla \ell((\tilde\xv, \tilde y); \wv) 
&= 
[(1-\lambda) \inner{\wv}{\gv(\mu)} + \lambda \EE_{(\xv,y) \sim \nu} l'(\wv^\top\xv, y) \inner{\wv}{\xv}] \tfrac{\gv(\mu)}{\inner{\wv}{\gv(\mu)}}
\\
&= \zero,
\end{align}
thanks to our choice of $\nu$. Applying \Cref{thm:main} again we know $\wv$ is $\lambda$-poisoning reachable. 

Conversely, if $\wv$ is $\lambda$-poisoning reachable, then from \Cref{thm:main} it follows that 
\begin{align}
    \zero \in (1-\lambda) \gv(\mu) + \lambda \EE_{(\xv,y)\sim\nu} l'(\wv^\top\xv , y) \xv.
\end{align}
Taking inner product with the model $\wv$ on both sides and noting that $\Gamma \subseteq \Pi$ we verify \eqref{eq:mp1}.
\end{proof}
\begin{remark}
The condition $\inner{\wv}{\gv(\mu)} \ne 0$ can be easily checked \emph{a priori}. In case it fails, two possibilities arise: 
\begin{itemize}
    \item $\gv(\mu) = \zero$, in which case poisoning is trivial: simply let $\nu = \mu$ for any $\lambda$.
    \item $\gv(\mu) \ne \zero$, in which case we may let $\nu$ concentrate on the line $L := \{\alpha \gv(\mu): \alpha \in \RR\}$. Thus, data poisoning succeeds if 
    \begin{align}
        0 = (1-\lambda) + \lambda \EE_{(\alpha, y)\sim \nu} l'(0, y) \alpha,
    \end{align}
    where we identify $\alpha \gv(\mu)$ as $\alpha$ for $\nu$. As long as $\Gamma$ contains some distribution that puts nonzero mass on $L$ and sufficiently large $l'(0,y)$, $\wv$ is again $\lambda$-poisoning reachable.
\end{itemize}
\label{rem:deg}
\end{remark}

Once we identify an appropriate subset $\Pi$ of poisoning distributions, we can even estimate the interval $\gfk$ using Monte Carlo algorithms.
Moreover, we may restrict the search of a poisoning distribution to the potentially much smaller subset $T_{\#}\Pi$ (where $\xv$ lies on the line spanned by $\gv(\mu)$). 

\begin{remark}[Connection to breakdown point] 
For simplicity consider $\Zds = \RR^d \times \RR$. 
It is well-known that unbounded convex losses $(t,y) \mapsto l(t-y)$, such as the square loss in \Cref{exm:ls}, have 0 breakdown point (and hence not robust): even adding a single poisoning point can perturb the model norm $\|\wv\|$ unboundedly \citep[e.g.][Theorem 5]{YuAS12}. 
\Cref{thm:lr} gives a much more detailed characterization: In fact, any target model $\wv$ can be induced by a diminishing amount of poisoning (even if $\nu$ is supported on a single point)! Indeed, since $l$ is unbounded and convex, there exists some $\tau\in\RR$ such that $| l'(\tau)| \ne 0$. It follows then $a = -\infty$ and $b = \infty$, and hence the threshold in \eqref{eq:linearthreshold} is trivially 0, for any target model $\wv$. 
Of course, our characterization in \Cref{thm:lr} continues to hold for any domain $\Zds$, unbounded or not.
\label{rem:breakdown}
\end{remark}

\begin{example}[Dichotomy]
    Consider the smooth loss\footnote{This is essentially a smoothed version of the perceptron loss $l(t) = \max\{-t, 0\}$.} 
    \begin{align}
    \label{eq:dich}
     l(t) = \begin{cases}
        -(4t+1)\exp(-2), & \mbox{ if } t \leq -\tfrac12 \\
        \exp(\tfrac1t), & \mbox{ if } t \in [-\tfrac12,0] \\
        0, & \mbox{ if } t \geq 0
        \end{cases}
        . 
    \end{align}
    Clearly, we have $a = 0$ and $b=\infty$. Thus, we arrive at a remarkable dichotomy: 
    \begin{itemize}
    \item If $\inner{\wv}{\gv(\mu)} = 0$ (in particular any separating $\wv$), then data poisoning succeeds with any $\lambda > 0$; 
    \item If $\inner{\wv}{\gv(\mu)} \ne 0$ (and hence $\wv$ cannot separate $\mu$), then data poisoning fails with any $\lambda < 1$.
    \end{itemize}
\end{example}
Note that $l$ in \eqref{eq:dich} is not calibrated since $l'(0) = 0$ \citep{BartlettJM06}, so it may not be a sensible loss to use in practice. For a calibrated margin loss $l$, \ie, one that is differentiable at 0 with $l'(0) < 0$, we necessarily have $ b > 0$ and $a < 0$, so the threshold in \eqref{eq:linearthreshold} usually lies strictly in $(0,1)$, incurring a nontrivial phase transitioning.

\begin{restatable}[Multiclass]{theorem}{LMC}
Consider $\ell(\xv, \yv; W) = l(W^\top\xv, \yv)$ for some loss $l$. 
Then\footnote{We use the notation $\av \otimes \bv := \av \bv^\top$ for two column vectors.}, 
\begin{align}
\label{eq:mc-gradw-app}
    G(\xv, \yv) := \nabla_{W}\ell(\xv, \yv; W) &= \xv \otimes \nabla l(W^\top\xv, \yv),
\end{align}
and $W$ is $\lambda$-poisoning reachable iff there exists $\nu \in \Gamma$ such that
\begin{align}
\label{eq:mc-grad-app}
    \zero \in (1-\lambda) G(\mu) + \lambda G(\nu).
\end{align}
Suppose $W^\top G(\mu)$ is non-degenerate and let 
\begin{align}
    \gfk := \{ \EE_{(\xv, \yv) \sim \nu} (W^\top\xv) \otimes \nabla l(W^\top\xv, \yv): \nu \in \Pi \}. \nonumber
\end{align}
Then, $W$ is $\lambda$-poisoning reachable
if $\Gamma \supseteq T_{\#}\Pi$ and 
\begin{align}
\label{eq:mc-mp}
    \zero \in (1-\lambda) W^\top G(\mu) + \lambda \gfk, 
\end{align}
where the transformation $T(\xv, \yv) := \left(G(\mu) [W^\top G(\mu)]^{-1} W^\top\xv , \yv \right)$.
Conversely, \eqref{eq:mc-mp} holds if $W$ is $\lambda$-poisoning reachable and $\Pi \supseteq \Gamma$.
\end{restatable}
\begin{proof}
The proof is completely similar to that of \Cref{thm:lm}.
\end{proof}

\begin{proof}
{[of \Cref{exm:ce}]} 
We aim to show that for any $\hv\in\RR^c$ and one-hot $\yv\in\RR^c$, we have 
\begin{align}
    -\Wfk(\tfrac{c-1}{e}) \leq \inner{\hv}{\pv - \yv} \leq \infty, \mbox{ where recall that } \pv := \softmax(\hv) = \exp(\hv) / \sum_k \exp(h_k).
\end{align}
The right-hand side is clear: we need only send some $h_k$ to $\infty$, as long as $y_k \ne 1$. For the left-hand side, we simplify as follows. W.l.o.g. assume $y_i = 1$. Then, 
\begin{align}
\inner{\hv}{\pv - \yv} = \sum_k h_k \left[\frac{\exp(h_k)}{\sum_j \exp(h_j)} - y_k\right] 
&= \frac{\sum_{k} (h_k - h_i) \exp(h_k-h_i)}{1+ \sum_{j\ne i} \exp(h_j - h_i)}  
\\
&= 
\sum_{k\ne i} \frac{\tfrac{1}{c-1} + \exp(h_k-h_i) }{1+ \sum_{j\ne i} \exp(h_j - h_i)} \cdot \frac{(h_k - h_i) \exp(h_k-h_i)}{\tfrac{1}{c-1} + \exp(h_k-h_i)} \\
\label{eq:LambertW}
&\geq \inf_t \frac{t \exp(t)}{\tfrac{1}{c-1} + \exp(t)} \\
&= -\Wfk(\tfrac{c-1}{e}),
\end{align}
where the inequality is achieved when $t \equiv h_k - h_i$ minimizes \eqref{eq:LambertW}.
\end{proof}

\NN*

\begin{proof}
It is straightforward to compute the gradients in \eqref{eq:nn-gradw} and \eqref{eq:nn-gradu}. The iff condition in \eqref{eq:nn-grad} then follows from \Cref{thm:main}. 
The necessary condition in \eqref{eq:nn-grad-W} is obtained by simply ignoring the second part of $G(\mu)$ (that corresponds to $\nabla_{\uv} \ell$). 
\end{proof}
From \eqref{eq:nn-grad-W} we conclude that the poisoning distribution $\nu$ must be supported at least on $s = \rank(G_1(\mu))$ points, as long as $\lambda \in (0,1)$.
Taking inner product \wrt $WA$ on both sides of \eqref{eq:nn-grad-W} we obtain
\begin{align}
\label{eq:nn-grad-uni}
    0 = (1-\lambda) g_A(\mu) + \lambda g_A(\nu),
\end{align}
where $g_A(\nu) = \EE_{(\xv, \yv) \sim \nu} \inner{A\nabla_{\hv} l(\hv, \yv)}{\hv}$ and $A$ is arbitrary. The condition \eqref{eq:nn-grad-uni} is univariate and easy to check, albeit being necessary but not sufficient. We remark that the free choice of the matrix $A$ may be exploited to tighten this necessary condition.

\section{Data poisoning as measure optimization}
\label{sec:measure}
In this section we discuss a measure optimization approach for solving the gradient canceling problem: 
\begin{align}
\label{eq:gc-app}
    \min_{\nu \in \Gamma} ~ \tfrac12\|\gv(\mu) + \epsilon_d \gv(\nu) \|_2^2,
\end{align}
where we recall that 
\begin{align}
\gv(\nu) = \EE_{\zv \sim \nu} \nabla_{\wv} \ell(\zv; \wv)
\end{align}
is the model gradient computed over the distribution $\nu$. 
The objective of \eqref{eq:gc-app} is a convex quadratic, although living in an infinite dimensional space (the vector space of all signed measures over $\Zds$). A particularly suitable way to solve \eqref{eq:gc-app} is the well-known Frank-Wolfe algorithm, where we repeatedly perform ``atomic'' updates to the measure $\nu$: 
\begin{align}
\label{eq:FW-eta}
\nu_{t+1} \gets (1-\eta_t)\nu_t + \eta_t \zeta_t,
\end{align}
where $\eta_t$ is the step size, \eg, $\eta_t = \tfrac{2}{t+2}$. The direction $\zeta_t$ is found by solving the linear minimization subproblem: 
\begin{align}
\min_{\zeta \in \Gamma} ~ \inner{\gv(\mu) + \epsilon_d \gv(\nu_t)}{\gv(\zeta)}.
\end{align}
When $\Gamma = \Pcal$ consists of all distributions over $\Zds$, the above subproblem simplifies to: 
\begin{align}
\label{eq:FW-sub}
\min_{\zv \in \Zds} ~ \inner{\gv(\mu) + \epsilon_d \gv(\nu_t)}{\nabla_{\wv} \ell(\zv; \wv)},
\end{align}
\ie, we find a new poisoning point $\zv$ to add to the support of the poisoning distribution $\nu_t$, while the step \eqref{eq:FW-eta} adjusts the probability mass. 
One particularly appealing part of this algorithm is that after $t$ iterations, the candidate poisoning distribution $\nu_t$ is supported at most on $t+1$ points (assuming we start with some $\nu_0$ supported on a single point). 
We remark that the subproblem \eqref{eq:FW-sub} is often nonconvex (in particular for neural networks), and could be challenging to solve. 
The other difficulty is that an attacker often is not allowed to upload an entire poisoning distribution, so a resampling procedure (on $\nu$) will be necessary to create a poisoning set, which is why we opted for a more direct approach in the main paper.

Another possibility is to parameterize $\nu$ as the push-forward of some fixed distribution (\eg, the training distribution), \ie, $\nu = [T(\theta)]_{\#}\mu$, and we optimize the push-forward transformation $T(\theta)$.

\section{Additional Experiments}
\label{app:add_exp}

\subsection{Additional implementation details}
\label{sec:exp-add}
\textbf{Hardware and package:} experiments were run on a cluster with \texttt{T4} and \texttt{P100} GPUs. The platform we use is PyTorch~\citep{PaszkeGMLBCKLGADKYDRTCSFBC19}. 

\textbf{Model in details}: for the MNIST dataset, we examine three target models: Logistic Regression; a neural network (NN) with three layers, where we choose hidden size as 784 and apply leaky ReLU with $\mathtt{negative\_slope}=0.2$ for activation; and a convolutional neural network (CNN) with two convolutional layers with kernel size 3, maxpooling and two fully connected layers with hidden size 128. 

\textbf{Synthetic Datasets:} in \Cref{fig:heatmap_or} and \Cref{fig:heatmap}, we perform experiemnts on two synthetic datasets. (1) OR dataset: we simply use the OR dataset in 2D space in \Cref{fig:OR} and repeat each point for 50 times (200 samples in total) with small Gaussian noise. (2) 10-D Gaussian dataset: we use the $\mathtt{make\_classification}$ function in $\mathtt{sklearn.datasets}$ with 1000 samples and 10 features.

\begin{figure*}[ht]
    \centering
    \includegraphics[width=0.3\textwidth]{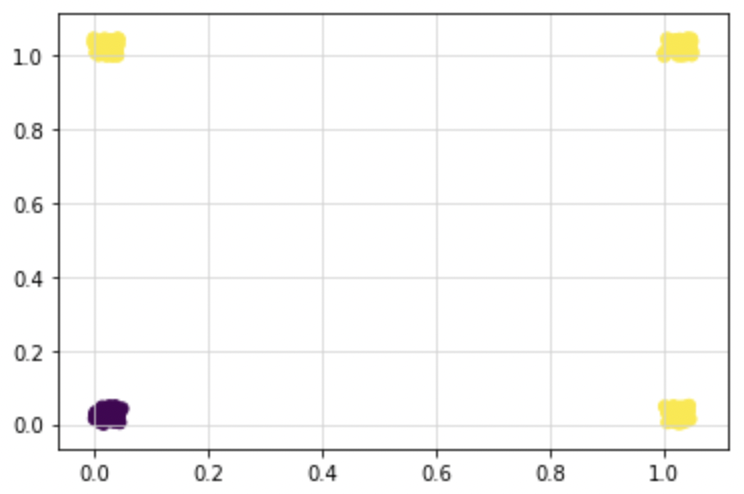}
    \caption{Here we visualize the OR dataset.} 
    \label{fig:OR}
\end{figure*}

\textbf{More on GradPC}: to choose proper target parameters (specifically, $\epsilon_w$), we use validation sets described in \Cref{sec:exp} for accuracy drop comparison. The GradPC attack never sees the test set during the construction of its perturbed models.

\textbf{Batch size:} for the Gradient Canceling experiments on MNIST, we set batch size as the size of entire training set (60000) for simplicity. For CIFAR-10 and TinyImageNet experiments, we set batch size as 1000 due to CUDA memory size constraint.

\textbf{Optimizer, learning rate scheduler and hyperparameters:}  we use SGD with momentum for optimization and the cosine learning rate scheduler \citep{loshchilov2016sgdr} for the Gradient Canceling algorithm. We set the initial learning rate as 0.5 and run 1000 epochs across every experiment.

\subsection{More on Parameter Corruption}

Recall in \Cref{tab:comparison} we compare TGDA with GradPC briefly. Here we show the complete results with more choices of $\epsilon_w$ and an additional baseline method called RandomPC in \citet{SunZRLL20} in \Cref{tab:comp2GradPC}.  
\begin{table*}[ht]
    \centering
    \small
    \caption{The attack accuracy/accuracy drop (\%) on the MNIST dataset. }
    \label{tab:comp2GradPC}    
    \setlength\tabcolsep{5pt}
    \scalebox{0.85}{\begin{tabular}{cccrrrrrr}
\toprule

\multirow{2}{*}[-.6ex]{\bf Target Model}
 & Clean & \bf TGDA & \multicolumn{3}{c}{\bf RandomPC} & \multicolumn{3}{c}{\bf GradPC}\\ 
\cmidrule(l{-1pt}r{2pt}){4-6}\cmidrule(l{2pt}r{0pt}){7-9}
 & Accuracy & Accuracy/Drop &$\epsilon_w=0.01$ & $\epsilon_w=0.1$ &$\epsilon_w=1$ & $\epsilon_w=0.01$ &$\epsilon_w=0.1$ & $\epsilon_w=1$ \\

\midrule
LR & 92.35 & 89.56 / 2.79 ($\epsilon_w =2.45$) & 91.94 / 0.41 &  81.24 / 11.11 &  24.66 / 67.69 & 91.91 / 0.44 & 89.72 / 2.63 & \bf 21.48 / 70.87 \\

NN & 98.04 & 96.54 / 1.50 ($\epsilon_w =0.55$) & 97.62 / 0.42 & 82.67 / 15.37 &  32.77 / 65.27 & 97.63 / 0.41 & 97.05 / 0.99 & \bf 31.14 / 66.90\\

CNN & 99.13 & 98.02 / 1.11 ($\epsilon_w =0.74$)  & 98.84 / 0.29 & 72.00 / 27.13 &  19.26 / 79.87 & 98.74 / 0.39 & 98.69 / 0.44 & \bf 12.98 / 86.15  \\

\bottomrule
\end{tabular}}
\end{table*}

\subsection{Least-square Regression}
\label{sec:ls}
\begin{figure*}[ht]
    \centering
    \includegraphics[width=1.0\textwidth]{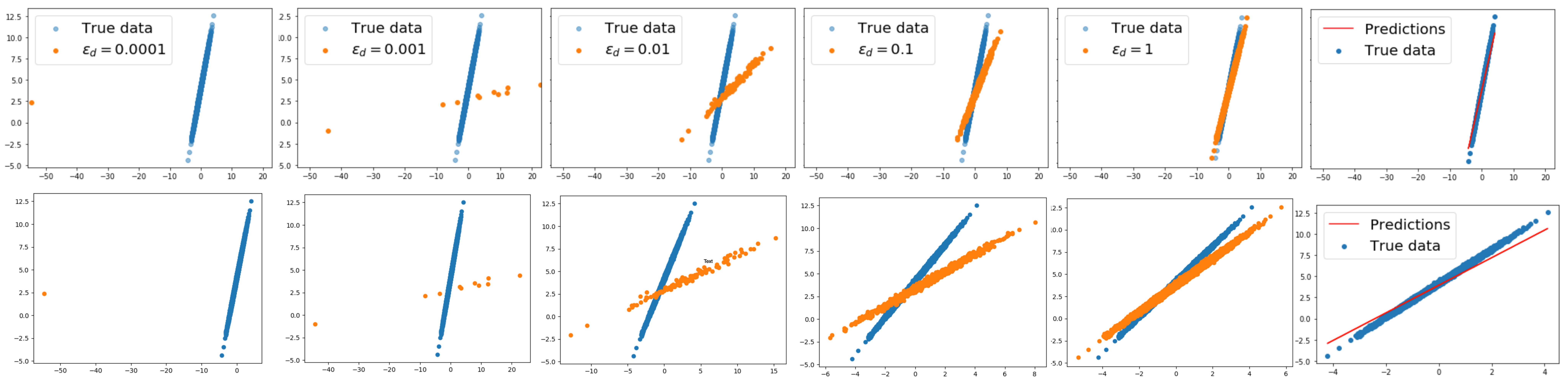}
    \caption{Here we run the Gradient Canceling algorithm on linear regression on a 2D Gaussian dataset. The first row displays all figures in the same scale for better comparison; and the second row shows the upper figures in their original scale for better viewing. (1) The left five figures show the poisoned points generated with different $\epsilon_d$. When $\epsilon_d$ is smaller, the poisoned points are farther from the data distribution. (2) The algorithm always generates the target parameter (the prediction) regardless of $\epsilon_d$.} 
    \label{fig:linear_gc}
\end{figure*}
 Recall that from \Cref{exm:ls} we conclude data poisoning with any $\epsilon_d > 0$ is possible for least-square regression. We perform GC attack on a synthetic 2D Gaussian dataset and visualize the results in \Cref{fig:linear_gc}. We observe that (1) the algorithm always generates the target parameter regardless of $\epsilon_d$, which immediately verifies our conclusion; (2) by increasing $\epsilon_d$,  the poison distribution $\nu$ gradually moves towards the data distribution $\mu$, which makes intuitive sense.
Moreover, recall that we may restrict the search of a poisoning distribution to the potentially much smaller subset $T_{\#}\Pi$ (where $\xv$ lies on the line spanned by $\gv(\mu)$), while in practice GC does not seem to always follow this theoretical construct.

\subsection{Comparison with Gradient Matching}

As we mentioned in \Cref{sec:gc}, one of the difference between Gradient Matching and our work is that there is no guarantee that after retraining over $\hat\nu$, gradient matching will arrive at the target model while our \Cref{alg:gc} explicitly aims to achieve this goal. We have shown in the \Cref{fig:linear_gc} that GC empirically achieve the target model regardless of $\epsilon_d$. By comparing with \Cref{fig:linear_gm}, we observe that gradient matching achieves different model parameters for every $\epsilon_d$. 

\begin{figure*}[ht]
    \centering
    \includegraphics[width=1.0\textwidth]{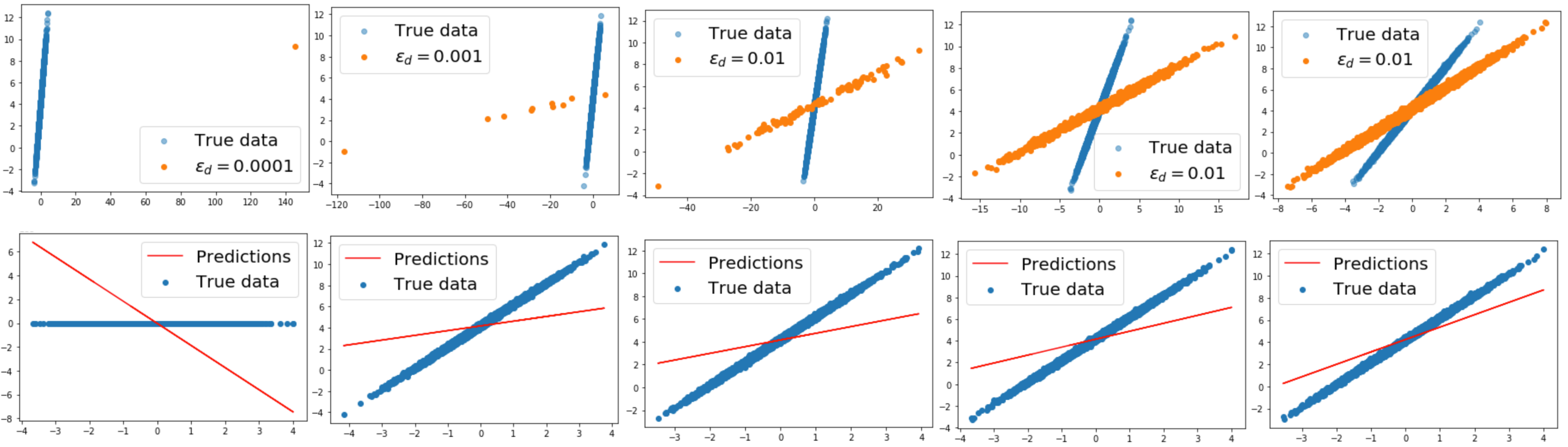}
    \caption{Here we run the Gradient Matching algorithm on linear regression on a 2D Gaussian dataset. The first row displays the poisoned points generated with different $\epsilon_d$; and the second row shows the different target parameters generated by the algorithm with different $\epsilon_d$.} 
    \label{fig:linear_gm}
\end{figure*}

\subsection{More on \Cref{fig:epsilon_w}}
Recall that in \Cref{fig:epsilon_w}, we fix $\epsilon_d$ and draw the learning curve for GC optimization for different $\epsilon_w$, where the $y$-axis indicates the normalized loss, \ie, $\|\gv(\chi)\|$. We observe that when $\tau > \epsilon_d$, $\|\gv(\chi)\|$ converges to a larger value, influenced by the distance between $\tau$ and $\epsilon_d$. 

Conversely, we fix $\epsilon_w$ (consequently, $\tau$) for different target models and repeat the MNIST experiments. In \Cref{fig:epsilon_p}, we again observe that when $\tau > \epsilon_d$, $\|\gv(\chi)\|$ converges at a relatively bigger number. Overall, we have confirmed the theoretical limitations proved in \Cref{sec:method}.
\begin{figure*}[hb]
    \centering
    \includegraphics[width=1.0\textwidth]{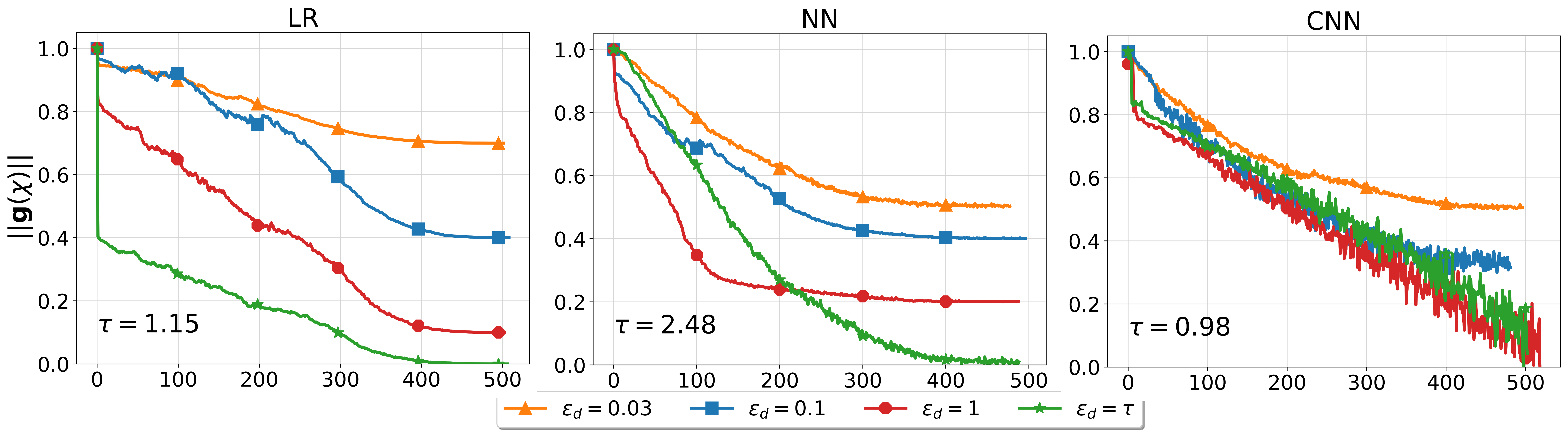}
    \caption{The learning curve for running GC on MNIST with different target models and $\epsilon_d$. Note that we fix $\epsilon_w$ for each model and print the respective $\tau$, and the loss indicates $\|\gv(\chi)\|$. The figure again confirms that GC cannot achieve $\wv$ if $\epsilon_d < \tau$.}
    \label{fig:epsilon_p}
\end{figure*}

\clearpage

\subsection{Scaling $\wv$}
We point out a subtlety in \Cref{example:log}: by scaling $\wv$ towards the origin, we do not change its accuracy (except the confidence it induces). However, the threshold $\tau$ tends to 0 and hence data poisoning succeeds in producing the target parameter $\wv$ with a smaller $\lambda$. In other words, less confident models are easier to poison to, which makes intuitive sense. For verification, we run the GC attack with scaled target parameter $\wv/2$ and compare it with the original target parameter $\wv$ in \Cref{tab:scale}. With the same target model accuracy, scaling $\wv$ significantly reduces its corresponding $\tau$, making it easier to poison to.

\begin{table}[ht]
    \centering
    \caption{The GC attack accuracy drop (\%) on  MNIST when scaling $\wv$ by half.}
    \label{tab:scale}    
    \setlength\tabcolsep{15pt}
    \scalebox{0.9}{\begin{tabular}{cccccl|rr}
\toprule

\bf Target Model & clean & GradPC  & $\tau(\wv)$ & $\tau(\wv/2)$\ & $\epsilon_d$ & \bf $\wv$ & \bf $\wv/2$  \\

\midrule
\multirow{3}{*}[-.6ex]{LR}  & \multirow{3}{*}[-.6ex]{92.35} & \multirow{3}{*}[-.6ex]{-70.87} & \multirow{3}{*}[-.6ex]{1.15} & \multirow{3}{*}[-.6ex]{0.54} & 0.03 & -22.97 & -44.11\\
& & & & &0.1 & -63.83 & -67.22  \\
& & & & & 1 & -67.01 & -79.55 \\
\midrule

\multirow{3}{*}[-.6ex]{NN} & \multirow{3}{*}[-.6ex]{98.04} & \multirow{3}{*}[-.6ex]{-20.03} & \multirow{3}{*}[-.6ex]{2.48} & \multirow{3}{*}[-.6ex]{1.41} & 0.03 & -6.10 & -9.29 \\
& & & & & 0.1& -9.77 & -11.01 \\
& & & & & 1&  -12.05 & -15.33 \\

\midrule

\multirow{3}{*}[-.6ex]{CNN} & \multirow{3}{*}[-.6ex]{99.13} & \multirow{3}{*}[-.6ex]{-24.78} & \multirow{3}{*}[-.6ex]{0.98} & \multirow{3}{*}[-.6ex]{0.42} & 0.03 & -9.55 & -12.03  \\
&&&&&0.1& -20.10 & -21.55 \\
&&&&& 1 & -23.80 & -24.56\\
\bottomrule
\end{tabular}}
\end{table}

The fact that simply scaling $\wv$ down could improve its poisoning reachability might seem surprising at first glance. However, this is due to a mismatch between how we train and how we test. It is best to explain this observation in the binary setting, where we note the mismatch between the common prediction rule
\begin{align}
    \hat{y} = \sign(\inner{\xv}{\wv}),
\end{align}
which is invariant to (positive) scaling (of $\wv$), and our training objective in finding a good parameter $\wv$, \eg, through logistic regression: 
\begin{align}
\label{eq:lr-app}
    \inf_{\wv} ~ \frac1n\sum_i \log [1+\exp(-y_i \inner{\xv_i}{\wv}],
\end{align}
which is not invariant to scaling (of $\wv$). 

Let us give an explicit example to further demonstrate this point. Consider three (cleaning) training points\footnote{This is in fact the smallest example: with 2 or fewer training points, logistic regression does not attain the infimum.} on the real line: 
\begin{align}
\label{eq:lr-toy}
    \xv_1 = \begin{bmatrix} 1 \\ 1\end{bmatrix}, y_1 = +; \quad 
    \xv_2 = \begin{bmatrix} -1 \\ 1\end{bmatrix}, y_2 = +; \quad 
    \xv_3 = \begin{bmatrix} 0 \\ 1\end{bmatrix}, y_3 = -; 
\end{align}
where we have padded 1 at the last entry of each $\xv$ (so that we can absorb the bias $b$ into $\wv$). Setting the derivative of \eqref{eq:lr-app} \wrt $\wv$ to zero we obtain: 
\begin{align}
    3 \cdot \gv(\wv) = -\frac{1}{1+\exp(w_1+w_2)}\begin{bmatrix}1\\ 1\end{bmatrix} - \frac{1}{1+\exp(w_2-w_1)} \begin{bmatrix} -1\\ 1\end{bmatrix} - \frac{1}{1+\exp(-w_2)} \begin{bmatrix}0\\ -1\end{bmatrix} = \zero.
\end{align}
Solving the above equation we have $\wv_\star = \begin{bmatrix} 0 \\ \ln 2\end{bmatrix}$. 

Now consider the scenario where we are given a target parameter $\wv = 2 \wv_\star$. According to our theory, the poisoning ratio 
\begin{align}
    \epsilon_d > \tau 
    &:\approx \max\left\{ 1.2 \left[\frac{-w_1-w_2}{1+\exp(w_1+w_2)} + \frac{w_1-w_2}{1+\exp(w_2-w_1)} + \frac{w_2}{1+\exp(-w_2)}\right], 0 \right\}
    \\
    &= \max\left\{ 1.2 w_2 \cdot \frac{\exp(w_2)-2}{1+\exp(w_2)}, 0 \right\}
    \\&= 0.48 \cdot 2\ln 2 \approx 0.67.
\end{align}
In other words, the poisoning set needs to be as large as 67\% of the training set, in order to produce $\wv = 2\wv_\star$. 
However, if we scale $\wv$ down to $\wv_\star$, then we do not even need to add any poisoned point, since $\wv_\star$ is already stationary (by definition). 
If we continue to scale $\wv$ down to say $0.5 \wv_\star$ (so that $\inner{\wv}{\gv(\wv)} < 0$ and hence $\tau = 0$), then for any $\epsilon > 0$, we may put $\epsilon$ copies of $\xv = \alpha \gv(\wv) \approx \tfrac{1}{\epsilon} \gv(\wv)$ as the poisoning set to produce $\wv$; see the detailed analysis below. 

More generally, consider adding $\epsilon$ copies of a poisoning data point at 
$
    \xv = \alpha \gv(\wv), y = +
$
(where $\alpha \in \RR$ will be determined later) 
so that the gradient on the clean and poisoned data is proportional to:
\begin{align}
    \gv(\wv) - \epsilon \frac{1}{1+\exp(\inner{\wv}{\xv})} \xv = \gv(\wv) - \epsilon \frac{1}{1+\exp(\inner{\wv}{\alpha\gv(\wv)})} \alpha\gv(\wv) = \gv(\wv) \cdot \left[1 - \epsilon \frac{\alpha}{1+\exp(\alpha \inner{\wv}{\gv(\wv)})}\right].
\end{align}
We can break the analysis into a few cases now: 
\begin{itemize}
    \item $\gv(\wv) = \zero$, \ie, we scale $\wv$ down to $\wv_\star$, in which case no poisoning point is needed to produce $\wv= \wv_\star$. 
    \item $\inner{\wv}{\gv(\wv)} = 0$, in which case the gradient reduces to $\gv(\wv) [ 1 - \epsilon \alpha /2 ]$. Therefore, for any $\epsilon > 0$, we may produce $\wv$ by putting $\epsilon$ copies of $\xv = \tfrac{2}{\epsilon} \gv(\wv)$.
    \item $\inner{\wv}{\gv(\wv)} < 0$, in which case for any $\epsilon > 0$, the function
    \begin{align}
        \alpha \mapsto 1 + \exp(\alpha \inner{\wv}{\gv(\wv)} - \epsilon \alpha
    \end{align}
    clearly has a zero $\alpha_*$ (easily seen by letting $\alpha \to \pm\infty$ and applying the intermediate value theorem). Thus again, we may produce $\wv$ by putting $\epsilon$ copies of $\xv$ at $\alpha_* \gv(\wv)$. (Note that $\alpha_* \to \infty$ if $\epsilon \to 0$; roughly $\alpha_* \approx \tfrac{1}{\epsilon}$.)
    \item $\inner{\wv}{\gv(\wv)} > 0$, in which case the function \begin{align}
        \alpha \mapsto 1 + \exp(\alpha \inner{\wv}{\gv(\wv)} - \epsilon \alpha
    \end{align}
    has a zero $\alpha_*$ iff $\epsilon \geq \tau$. Indeed, 
    \begin{align}
        1 = \epsilon \frac{\alpha}{1+\exp(\alpha \inner{\wv}{\gv(\wv)})} \iff 1 &= \frac{\epsilon}{\inner{\wv}{\gv(\wv)}} \cdot \frac{\alpha\inner{\wv}{\gv(\wv)}}{1+\exp(\alpha \inner{\wv}{\gv(\wv)})} 
        \\
        &\leq \frac{\epsilon}{\inner{\wv}{\gv(\wv)}} \cdot \sup_t \frac{t}{1+\exp(t)} 
        \\
        &= \frac{\epsilon}{\inner{\wv}{\gv(\wv)}} \cdot \Wcal(1/e).
    \end{align}
    Thus again, for any $\epsilon \geq \tau$, we may produce $\wv$ by putting $\epsilon$ copies of $\xv$ at $\alpha_* \gv(\wv)$.
\end{itemize}

In \Cref{fig:scaling} we observe that GC converged to \emph{nonzero} loss (i.e., unable to produce the target parameter) when $\epsilon_d < \tau$, for any learning rate we tried, while after scaling $\mathbf{w}$ down so that $\epsilon_d > \tau$, GC immediately converged to zero loss (without the need of tuning the learning rate), confirming our theoretical analysis above. We plan to further explore the scaling effect in future work. 

\begin{figure*}[htb]
    \centering
    \includegraphics[width=0.5\textwidth]{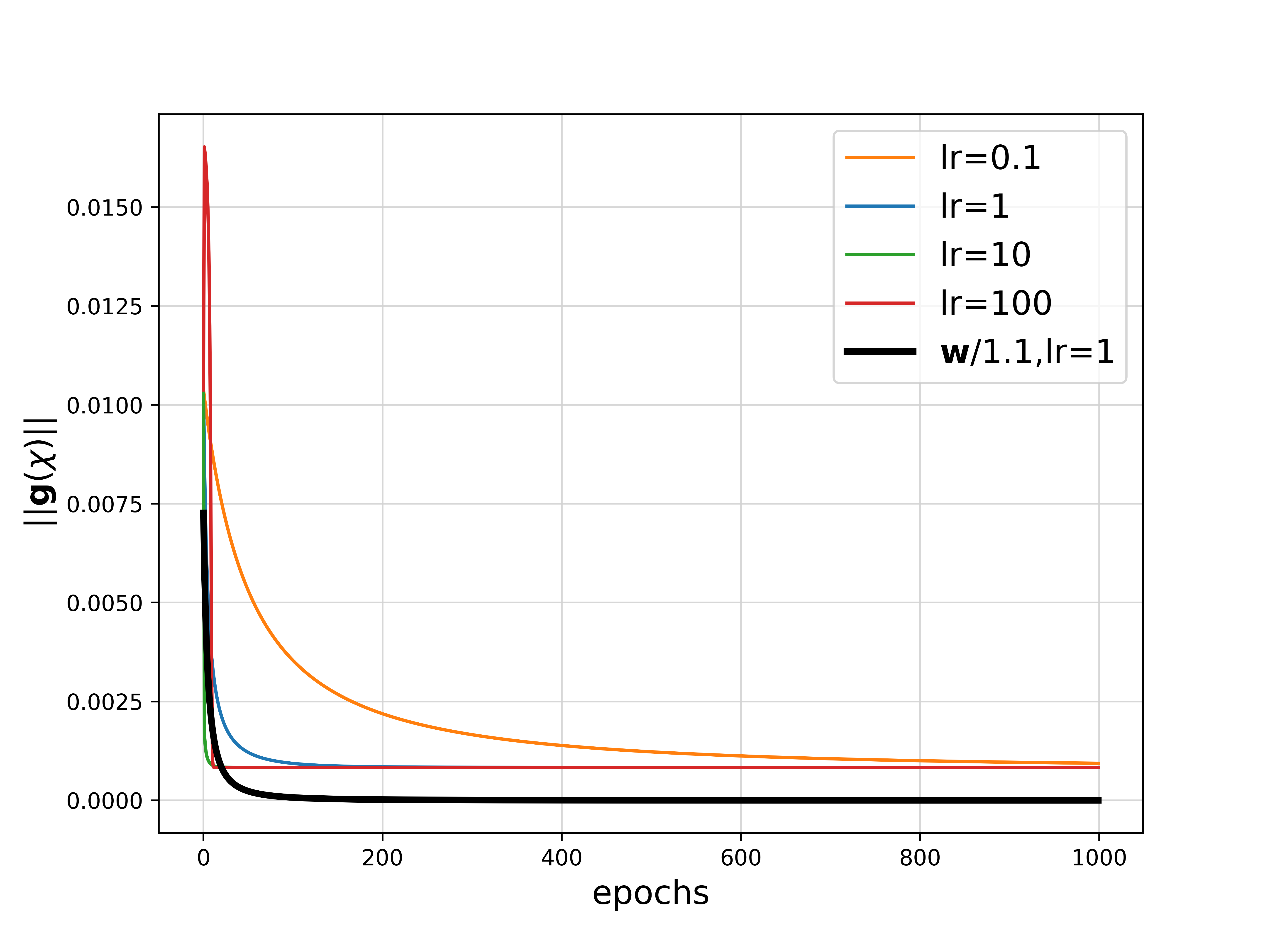}
    \caption{We plot the training curves of GC on the toy example (see \eqref{eq:lr-toy}). The colored curves represent GC attack on $\wv = \begin{bmatrix} 0 \\ 2\ln 2\end{bmatrix}$ with  $\epsilon=0.52<\tau\approx0.67$ under different learning rates; the black curve represents GC attack on the scaled target parameter $\wv/1.1 \approx \begin{bmatrix} 0 \\ 1.82\ln 2\end{bmatrix}$ with $\epsilon=0.52 > \tau \approx 0.51$ under learning rate 1.  } 
    \label{fig:scaling}
\end{figure*}

\subsection{Simulating Different Target Parameters}
Next, we verify if GC can achieve any desired target parameter. We choose poisoned models generated by TGDA attack as target parameters and perform PC. We discover that such parameters are also achievable by GC in \Cref{tab:simulate}, which further confirms that GC may be equipped with any other parameter corruption methods, regardless of how the target parameters are generated.

\begin{table}[ht]
    \centering
    \caption{Simulating TGDA attack ($\epsilon_d=1$) with Gradient Canceling attack on the MNIST dataset.}
    \label{tab:simulate}    
    \setlength\tabcolsep{15pt}
    \scalebox{0.9}{\begin{tabular}{cccclr}
\toprule

\bf Target Model & clean & TGDA  & $\tau$  & $\epsilon_d$ & \bf GC\\

\midrule
\multirow{4}{*}[-.6ex]{LR}  & \multirow{4}{*}[-.6ex]{92.35} & \multirow{4}{*}[-.6ex]{-8.97} & \multirow{4}{*}[-.6ex]{2.33} & 0.03 & -2.66 \\
& & & &0.1 & -3.39  \\
& & & & 1 & -5.53 \\
& & & & $\tau$ & -8.35 \\
\midrule
\multirow{3}{*}[-.6ex]{NN}  & \multirow{3}{*}[-.6ex]{98.04} & \multirow{3}{*}[-.6ex]{-5.49} & \multirow{3}{*}[-.6ex]{0.95} & 0.03 & -1.39 \\
& & & & 0.1 & -1.55 \\
& & & & 1 & -4.99 \\
\midrule

\multirow{3}{*}[-.6ex]{CNN}  & \multirow{3}{*}[-.6ex]{99.13} & \multirow{3}{*}[-.6ex]{-4.76} & \multirow{3}{*}[-.6ex]{0.49} & 0.03 & -0.98 \\
& & & & 0.1 & -2.10 \\
& & & & 1 & -4.68 \\

\bottomrule
\end{tabular}}
\end{table}

\subsection{GC against Defenses}
\label{sec:more-defenses}
Next, we examine the GC attack against three popular distribution-wise defenses. (1) Influence defense \citep{KL17} removes $\epsilon_d$ suspicious points according to higher influence functions; (2) Sever \citep{DiakonikolasKKLSS19} removes $\epsilon_d$ training points with the highest outlier scores, defined using the top singular value in the matrix of gradients; (3) Maxup defense \citep{GongRYL20} generates a set of augmented data with random perturbations and then aims at minimizing the worst case loss over the augmented data. 

We present our results on the MNIST dataset in \Cref{tab:defense_full} and observe that: (1) Among the three defenses, Sever is the most effective one and can significantly reduce the effectiveness of GC.
(2) Clipping the poisoned data to the range of clean training set makes GC more robust against all defenses, with the tradeoff of attack effectiveness.
(3) Larger $\epsilon_d$ makes the attack generally more robust, which matches our observation in least-squared regression.
\begin{table*}[ht]
\addtolength{\tabcolsep}{-3pt}
    \centering
    \caption{The accuracy drop (\%) Gradient Canceling attack (w/wo clipping) introduces on MNIST with Influence/Sever/MaxUp defense (+ indicates the accuracy increased by defenses). GC: original Gradient Canceling attack; GC-c: GC with clipped output; GC-d: GC after defense; GC-cd: GC-c after defense.}
    \label{tab:defense_full}    
    \setlength\tabcolsep{5pt}
    \scalebox{0.8}{\begin{tabular}{rrrrrrrrrrr}
\toprule

\multirow{2}{*}[-.6ex]{\bf Model} & \multirow{2}{*}[-.6ex]{Clean Acc}& \multirow{2}{*}[-.6ex]{$\epsilon_d$} &\multirow{2}{*}[-.6ex]{GC} & \multirow{2}{*}[-.6ex]{GC-c} &\multicolumn{2}{c}{Influence} & \multicolumn{2}{c}{Sever} & \multicolumn{2}{c}{MaxUp} \\
 
\cmidrule(l{0pt}r{10pt}){6-7}\cmidrule(l{0pt}r{10pt}){8-9}\cmidrule(l{0pt}r{10pt}){10-11}
& & & & & GC-d & GC-cd & GC-d  & GC-cd & GC-d  & GC-cd\\

\midrule
\multirow{3}{*}[-.6ex]{LR} & \multirow{3}{*}[-.6ex]{92.35} & 0.03 &-22.79 &  -11.28 & -21.99 / +0.80  & -11.17 / +0.11 & -12.81 / \bf +9.98 & -9.66 / \bf +1.62 & -22.59 / +0.20 & -11.26 / +0.02 \\

& & 0.1 & -63.83 & -26.77 & -63.51 / +0.32 & -26.67 / +0.10 & -59.79 / \bf +4.04 & -25.53 / \bf +1.24 & -63.65 / +0.18 & -26.67 / +0.10 \\

& & 1 & -67.01 & -28.99 & -66.75 / +0.26 & -26.71 / +0.06  & -65.01 / \bf +2.00 & -27.89 / \bf +1.10 & -66.02 / +0.09 & -28.97 / +0.02 \\

\midrule

\multirow{3}{*}[-.6ex]{NN} & \multirow{3}{*}[-.6ex]{98.04} & 0.03 & -6.10 & -3.25 & -5.59 / +0.51 & -3.16 / +0.09 & -3.22 / \bf +2.88 & -2.26 / \bf +0.90  & -6.08 / +0.02 & -3.24 / +0.01 \\
& & 0.1 & -9.77 & -5.10 & -9.32 / +0.45 & -5.02 / +0.08 & -7.66 / \bf +2.11 & -4.46 / \bf +0.56 & -9.76 / +0.01 & -5.10 / +0.00\\
& & 1 & -12.05 & -6.53 &  -11.65 / +0.40 & -6.48 / +0.05 & -10.02 / \bf +2.03 & -6.11 / \bf +0.42 & -12.04 / +0.01 & -6.53 / +0.00\\
\midrule

\multirow{3}{*}[-.6ex]{CNN} & \multirow{3}{*}[-.6ex]{99.13}  & 0.03 & -9.55 & -5.87 & -8.57 / +0.98 & -5.56 / +0.31 & -5.55 / \bf +4.00 & -4.36 / \bf +1.51 & -9.39 / +0.16 & -5.83 / +0.04 \\
& & 0.1 & -20.10 & -12.50 & -19.19 / +0.91 & -12.35 / +0.15 & -16.55 / \bf +3.55 & -11.32 / \bf +1.18 & -20.06 / +0.04 & -12.48 / +0.02 \\
& & 1 & -23.80 & -13.32 & -23.10 / +0.70 & -13.21 / +0.11 & -21.05 / \bf +2.75 & -12.51 / \bf +0.81 & -23.79 / +0.01 & -13.32 / +0.00\\

\bottomrule
\end{tabular}}
\end{table*}

\subsection{Visualization of Poisoned Images}

Finally, we visualize some poisoned images generated by the GC attack in \Cref{fig:mnist} and \Cref{fig:cifar}.

\begin{figure*}[htb]
    \centering
    \includegraphics[width=1.0\textwidth]{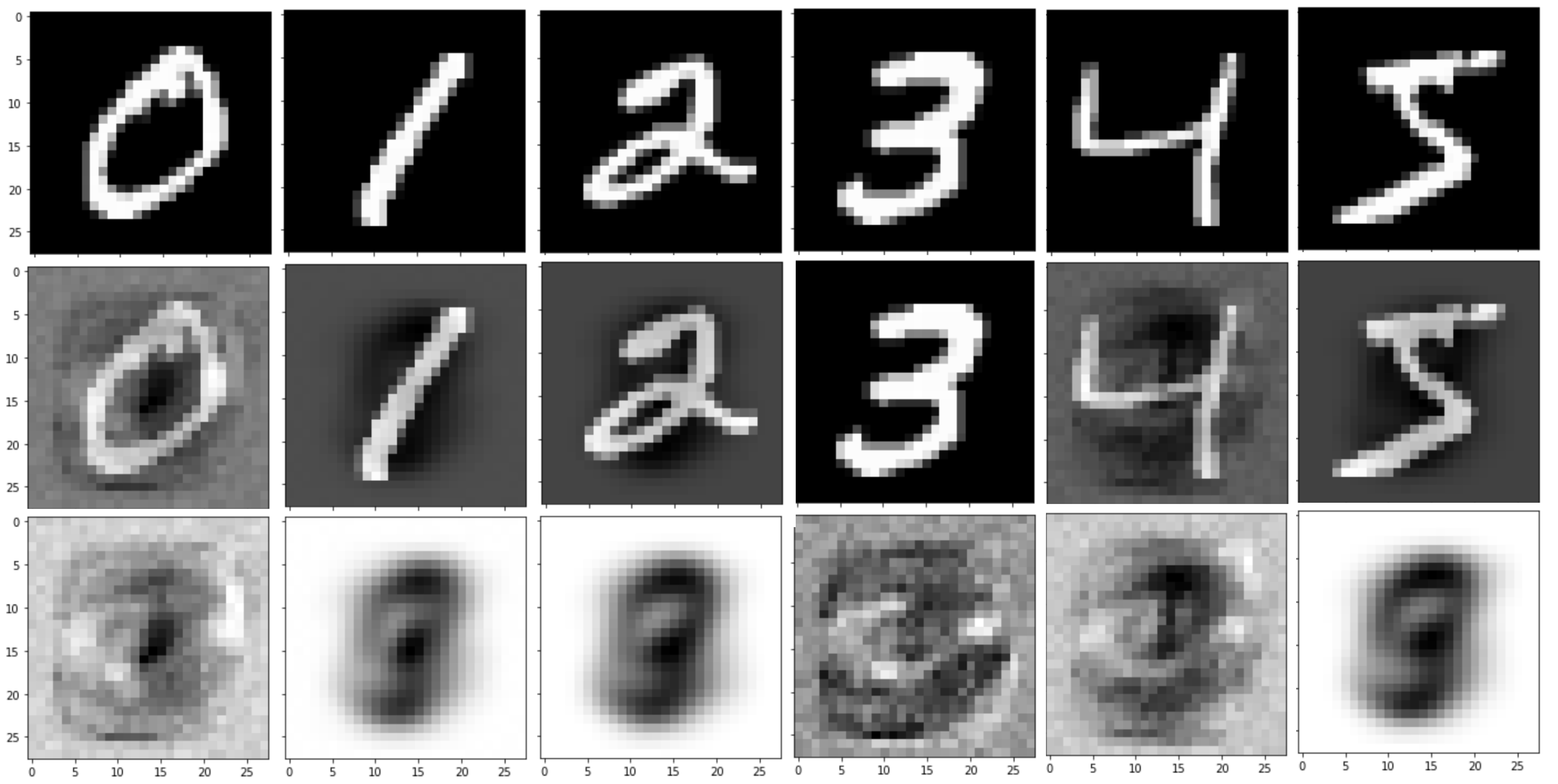}
    \caption{We visualize some poisoned images generated by the GC attack on the MNIST dataset. The first row shows the clean samples, the second row shows the poisoned samples; the third row displays the perturbation.} 
    \label{fig:mnist}
\end{figure*}

\begin{figure*}[htb]
    \centering
    \includegraphics[width=1.0\textwidth]{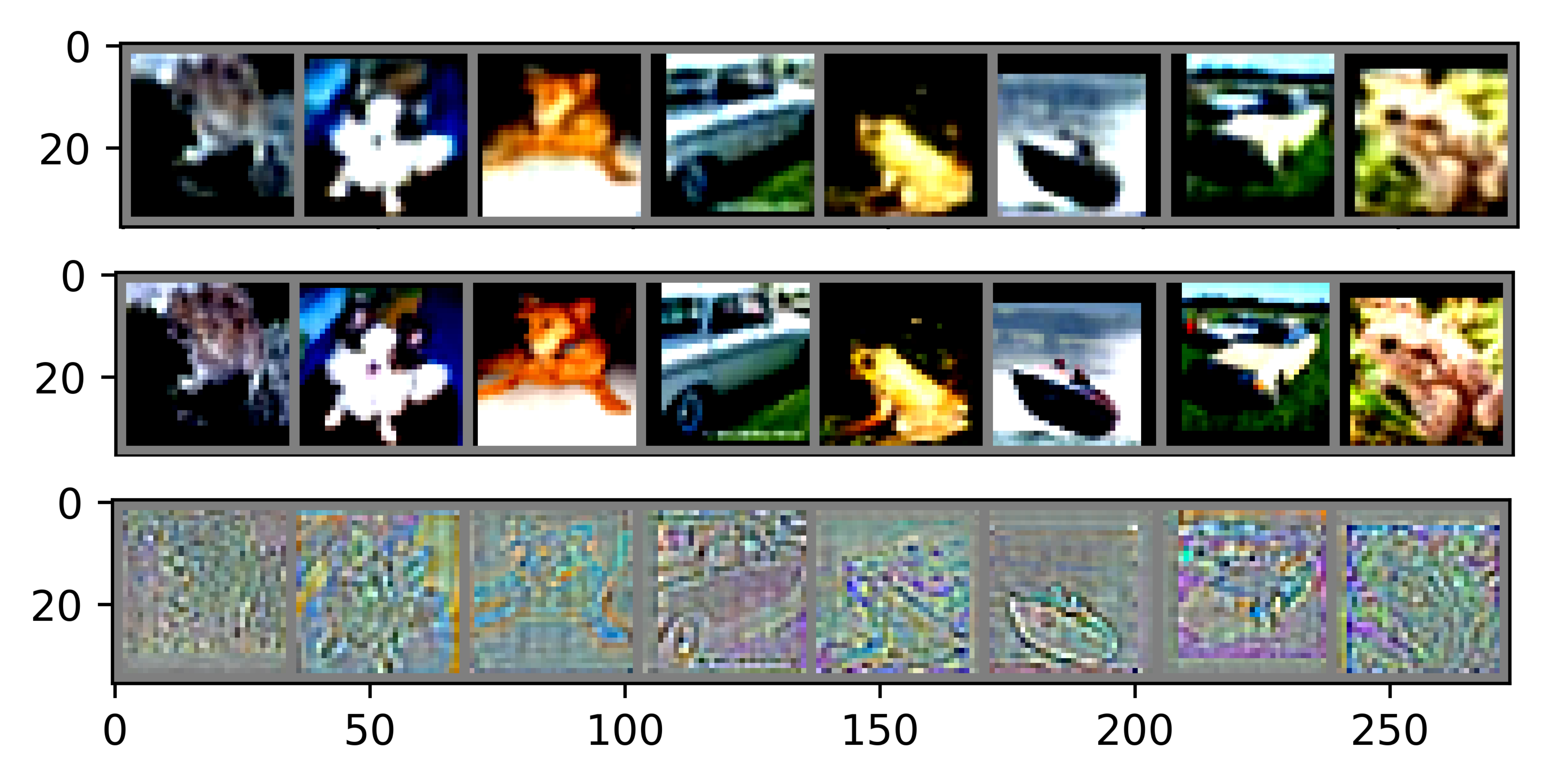}
    \caption{We visualize some poisoned images generated by the GC attack on the CIFAR-10 dataset. The first row shows the clean samples, the second row shows the poisoned samples; the third row displays the perturbation.} 
    \label{fig:cifar}
\end{figure*}
\clearpage

\subsection{Comparison with Replacing Attack}
\label{app:replace}

In this work we only consider an adversary who is restricted to \emph{add} corrupted points $\Dpo$ to the intact (clean) training set $\Dtr$, while an even stronger attacker might consider replacing part of $\Dtr$ with $\Dpo$ (also closely related to the nasty noise model \cite{BshoutyEK02}). We first formulate the general case: recall that we consider the mixed distribution $\chi = (1-\lambda)\mu + \lambda \nu$ of the clean distribution $\mu$ and poisoned distribution $\nu$, where $\lambda$ is the proportion of poisoning data. Then, replacing part of the clean training data is equivalent to:
\begin{align}
    \chi = (1-\lambda')\mu' + \lambda' \nu,
\end{align}
where $\mu'$ is a subset of $\mu$, and $\lambda'=|\Dpo|/|\Dtr|$. Empirically, with the ability to replace data points we may still apply Gradient Canceling in a straightforward manner: the only difference is that in Algorithm 1 we change $\mu$ to $\mu'$, a random subset of $\mu$. Following this idea, we perform a simple experiment: we choose $\epsilon_d=0.03$, and choose $\mu'$ to be a random subset of $\mathcal{D}_{tr}$, with size $\tfrac{1}{1+\epsilon_d}|\mathcal{D}_{tr}| \approx 0.97|\mathcal{D}_{tr}|$.  The results on MNIST are presented  below in \Cref{tab:replace}:

\begin{table}[ht]
    \centering
    \caption{Gradient Canceling attack ($\epsilon_d=0.03$) with adding-only vs replacing-only on the MNIST dataset.}
    \label{tab:replace}    
    \setlength\tabcolsep{15pt}
    \scalebox{0.9}{\begin{tabular}{cccc}
\toprule

\bf Target Model & clean & \bf GC (adding-only) & \bf GC (replacing)\\

\midrule
LR & 92.35 & -22.97 & -23.10\\
\midrule
NN &98.04 & -6.10 & -6.35\\
\midrule
CNN & 99.13 & -9.55 & -9.62\\

\bottomrule
\end{tabular}}
\end{table}

We observe that the ability to replace clean training data is indeed (slightly) more powerful than the corresponding adding-only attack. Notably, we remove training samples randomly, which may not be relatively weak. Ideally, an adversary would remove the most important points (e.g., in \cite{IlyasPELM22}) to further reduce the test accuracy. This improved replacing attack might be worth future exploration, although we note that it is less likely to be applicable when an attacker does not have direct access to a victim's infrastructure.

\section{Selecting Target Parameters}
\label{app:tar}
Here we discuss how to select an appropriate target parameter $\wv$ for the GC attack. In principle, there are two major factors regarding the selection of target parameters $\wv$:
(1) strength of $\wv$, measured by the test accuracy drop it incurs;
(2) poisoning reachability, measured by the optimality condition (i.e., the empirical loss in \Cref{eq:loss_function}).
We want to choose a $\wv$ that is both reachable and as strong as possible. Next, we discuss both criteria in details:
\begin{itemize}
    \item \textbf{Strength of $\wv$}: (a) existing works (e.g., \cite{KohSL18,SuyaMSET21}) only explored rudimentary ways to construct target parameters (e.g., through the label flip attack), and thus are less effective in casting particularly powerful target parameters; (b) with GradPC, we can now easily quantify the strength of a target parameter $\wv$ using $\epsilon_w$ (in Table 1); 
    (c) thus in practice, we first prepare a sequence of target parameters $\{\wv_k: k \in K\}$ with $|K|$ different $\epsilon_w$, and then send them all to the reachability test (in the next step).
    \item \textbf{Reachability test}: 
    (a) given the list of $\{\wv_k: k \in K\}$, we first calculate every corresponding $\tau(\wv_k)$, such that we can already rule out a few choices where $\epsilon_d < \tau$ (that we know GC cannot achieve with the existing budget $\epsilon_d$). After this process, we only keep a subset of target parameters $\{\wv_k: k \in \bar K\}$;
    (b) next, we run GC for each $\{\wv_k : k \in \bar K\}$, and examine if GC can achieve them by checking the loss upon convergence. We only keep those $\wv_k$'s that return a loss smaller than a margin (this margin is defined by one-tenth of the initial loss) when $\epsilon_d\approx\tau$.
    (c) finally, we empirically select a target parameter $\wv$ with the largest accuracy drop on the validation set.
\end{itemize}

\clearpage
\printbibliography[title={Extra References},filter=appendixOnlyFilter]
\end{document}